\documentclass{article}

 \PassOptionsToPackage{numbers, compress}{natbib}

 \usepackage[final]{neurips_2018}

\usepackage{lipsum}
\usepackage{enumitem}
\setlist[description]{leftmargin=3mm}
\usepackage[utf8]{inputenc} 
\usepackage[T1]{fontenc}    
\usepackage{hyperref}       
\usepackage{url}            
\usepackage{booktabs}       
\usepackage{amsfonts}       
\usepackage{nicefrac}       
\usepackage{microtype}      

\usepackage{xcolor}
\usepackage{dsfont}
\usepackage{multicol}

\newcommand{\cX}{\mathcal{X}}
\newcommand{\ed}{\stackrel{\mathrm{def}}{=}}
\newcommand{\pf}{f(p)}
\newcommand{\indic}{\mathds{1}}
\newcommand{\Poi}{\mathrm{Poi}}
\newcommand{\Paren}[1]{\left(#1\right)}
\newcommand{\EE}{\mathbb{E}}
\newcommand{\mul}[1]{N_{#1}}
\newcommand{\ignore}[1]{}
\makeatletter

\newcommand{\lipf}{\ell_f}

\newcommand{\umax}{u_{\max}}
\newcommand{\PP}{\mathbb{P}}
\newcommand{\Var}{\mathrm{Var}}

\usepackage{amsmath,amssymb,amsfonts,amsthm,bbm,graphicx}
\newtheorem{theorem}{Theorem}
\newtheorem{Lemma}{Lemma}

\title{Data Amplification:
A Unified and Competitive\\ Approach to Property Estimation}

\author{
  Yi~Hao\\
  Dept. of Electrical and Computer Engineering \\
 University of California, San Diego\\
 La Jolla, CA 92093\\
  \texttt{yih179@eng.ucsd.edu} \\
 \And
Alon~Orlitsky \\
 Dept. of  Electrical and Computer Engineering\\
 University of California, San Diego\\
 La Jolla, CA 92093\\
  \texttt{alon@eng.ucsd.edu} \\
  \And
  Ananda T.~ Suresh\\
  Google Research, New York\\
  New York, NY 10011\\
  \texttt{theertha@google.com}\\
  \And
  Yihong~Wu\\
Dept. of Statistics and Data Science\\
Yale University\\
New Haven, CT 06511\\
\texttt{yihong.wu@yale.edu}\\
}

\begin{document}

\maketitle

\begin{abstract}
Estimating properties of discrete distributions is a fundamental problem in statistical learning. We design the first unified, linear-time, competitive, property estimator that for a wide class of properties and for all underlying distributions uses just $2n$ samples to achieve the performance attained by the empirical estimator with $n\sqrt{\log n}$ samples. This provides off-the-shelf, distribution-independent, ``amplification'' of the amount of data available relative to common-practice estimators. 

We illustrate the estimator's practical advantages by comparing it to existing estimators for a wide variety of properties and distributions. In most cases, its performance with $n$ samples is even as good as that of the empirical estimator with $n\log n$ samples, and for essentially all properties, its performance is comparable to that of the best existing estimator designed specifically for that property.
\end{abstract}

\section{Distribution Properties}\label{background}
Let $D_{\cX}$ denote the collection of distributions over 
a countable set $\cX$ of finite or infinite cardinality $k$.
A distribution \emph{property} is a mapping $f:D_{\cX}\to\mathbb{R}$.
Many applications call for estimating properties of an unknown
distribution $p\in D_\cX$ from its samples.
Often these properties are \emph{additive}, namely can be written as a
sum of functions of the probabilities. 
Symmetric additive properties can be written
as
\[
\pf \ed \sum_{x \in \cX} f(p_x),
\]
and arise in many biological, genomic, and language-processing applications:
\begin{description}
\item[Shannon entropy]
$\sum_{x \in \cX} p_x\log \frac{1}{p_x}$, where throughout the paper
$\log$ is the natural logarithm,
is the fundamental information measure arising in a variety of applications~\cite{info}. 
\item[Normalized support size]
$\sum_{x \in \cX}\frac{1}{k}\indic_{p_x>0}$
plays an important role in population~\cite{population} and vocabulary size estimation~\cite{vocabulary}. 
\item[Normalized support coverage]
$\sum_{x \in \cX}\frac{1-e^{-mp_x}}{m}$
is the normalized expected number of distinct elements observed upon
drawing $\Poi(m)$
independent samples, it arises in ecological~\cite{ecological}, genomic~\cite{genomic}, and database studies~\cite{database}. 
\item[Power sum] $\sum_{x \in \cX}p_x^a$, arises
in R\'enyi entropy~\cite{renyientropy}, Gini impurity~\cite{gini}, and related diversity measures. 
\item[Distance to uniformity]
$\sum_{x \in \cX}\left|p_x-\frac1k\right|$,
appears in property testing~\cite{testingu}.
\end{description}
More generally, non-symmetric additive properties can be expressed as
\[
\pf \ed \sum_{x \in \cX} f_x(p_x),
\]
for example distances to a given distribution, such as:
\begin{description}
\item[L1 distance]
$\sum_{x \in \cX}\left|p_x-q_x\right|$,
the $L_1$ distance of the unknown distribution $p$ from a given
distribution $q$, 
appears in hypothesis-testing errors~\cite{testing}.
\item[KL divergence]
$\sum_{x \in \cX} p_x\log\frac{p_x}{q_x}$, 
the KL divergence of the unknown distribution $p$ from a given
distribution $q$, 
reflects the compression~\cite{info} and prediction~\cite{kl} 
degradation when estimating $p$ by $q$. 
\end{description}
Given one of these, or other, properties, we would like to estimate
its value based on samples from an underlying distribution. 
\section{Recent Results}
In the common property-estimation setting, the unknown distribution
$p$ generates $n$ i.i.d. samples $X^n\sim p^n$, which in turn are
used to estimate $f(p)$. 
Specifically, given property $f$, we 
would like to construct an \emph{estimator} 
$\hat{f}: \cX^*\to\mathbb{R}$
such that $\hat{f}(X^n)$ is as close to $f(p)$ as possible.
The standard \emph{estimation loss} is the expected squared loss
\[
\EE_{X^n\sim p^n}  \Paren{\hat{f}(X^n) - \pf}^2.
\]

Generating exactly $n$ samples creates dependence between the
number of times different symbols appear. To avoid these dependencies
and simplify derivations, we use the well-known
\emph{Poisson sampling}~\cite{poisamp} paradigm.
We first select $N\sim\Poi(n)$, and then generate $N$ independent
samples according to $p$. 
This modification does not change the statistical nature of the
estimation problem since a Poisson random variables is exponentially
concentrated around its mean. Correspondingly the estimation loss is
\[
L_{\hat{f}}(p, n)
\ed
\EE_{N\sim\Poi(n)}\left[\EE_{X^N\sim p^N}\left(\hat{f}(X^N) - \pf \right)^2\right].
\]

For simplicity, let $\mul{x}$ be
the number of occurrences of symbol $x$ in $X^n$. An intuitive
estimator is the plug-in \emph{empirical estimator} $f^E$ that
first uses the $N$ samples to estimate $p_x=N_x/N$ and then estimates
$\pf$ as 
\ignore{
\[
f^E(X^N) \ed \sum_{x \in \cX} f\Paren{\frac{\mul{x}}{N}},
\]
where $f^E(X^0)\ed 0$.
}
\[
f^E(X^N)
\ed
\begin{cases}
\sum_{x \in \cX} f_x\Paren{\frac{\mul{x}}{N}} & N>0,\cr
0 & N=0.
\end{cases}
\]

Given an error tolerance parameter $\delta>0$,
the \emph{$(\delta,  p)$-sample complexity} of an estimator $\hat{f}$
in estimating
$f(p)$ is the smallest number of samples $n$ allowing for estimation
loss smaller than $\delta$,
\[
n_{\hat{f}}(\delta, p) \ed \min_{n\in \mathbb{N}}\{L_{\hat{f}}(p, n)<\delta\}.
\]

Since $p$ is unknown, the common \emph{min-max} approach 
considers the worst case $(\delta, p)$-sample complexity
of an estimator $\hat{f}$ over all possible $p$,
\[
n_{\hat{f}}(\delta)
\ed
\max_{p\in D_{\cX}}{n_{\hat{f}}(\delta, p)}.
\]
Finally, the estimator minimizing $n_{\hat{f}}(\delta)$
is called the \emph{min-max estimator} of property $f$, denoted $f^{\rm M}$.
\ignore{
Finally, let
\[
n^*(\delta)
\ed
\min_{\hat f} n_{\hat{f}}(\delta)
\]
be the lowest number of samples required to estimate $f(p)$ in the
worst case, and let $f^*$ be the estimator achieving this number of
samples, also called the min-max estimator for $f$. 
}
It follows that $n_{f^M}(\delta)$ is the smallest Poisson
parameter $n$, \vspace{-0.33em}or roughly the number of samples, needed for any estimator $\hat{f}$ 
to estimate $f(p)$ to estimation loss $\delta$ for all $p$.

There has been a significant amount of recent work on property estimation.
In particular, it was shown that for all seven properties mentioned earlier, 
$f^M$ improves the sample complexity by a logarithmic factor compared to $f^E$. 
For example, for Shannon entropy~\cite{mmentro},
normalized support size~\cite{mmsize},
normalized support coverage~\cite{mmcover}, and distance to uniformity~\cite{mml1},
$n_{f^E}(\delta)=\Theta_{\delta}(k)$ while
$n_{f^M}(\delta)=\Theta_{\delta}(k/\log k)$.
Note that for normalized support size, $D_{\cX}$ is typically
replaced by $D_k:= \{p\in D_{\cX}: p_x\geq1/k,\forall x\in \cX \}$,
and for normalized support coverage, $k$ is replaced by $m$.
\vspace{-0.25em}
\section{New Results}\label{newr}
\vspace{-0.35em}
While the results already obtained are impressive, they also have some
shortcomings. 
Recent state-of-the-art estimators are designed~\cite{mmentro, mmsize, mml1}
or analyzed~\cite{mmcover, valiant} to estimate each individual property. 
Consequently these estimators cover only few properties. 
Second, estimators proposed for more general properties~\cite{mmcover,jnew}
are limited to symmetric properties and are not known
to be computable in time linear in the sample size.
Last but not least, by design, min-max estimators are optimized
for the ``worst'' distribution in a class.
In practice, this distribution is often very different,
and frequently much more complex, than the actual underlying
distribution.
This ``pessimistic'' worst-case design results in sub-optimal
estimation, as born by both the theoretical and experimental results. 

In Section~\ref{newf}, we design an estimator $f^*$ that addresses all these issues. 
It is \emph{unified} and applies to a wide
range of properties, including all previously-mentioned properties
($a>1$ for power sums) and all \emph{Lipschitz properties} $f$ where
each $f_x$ is Lipschitz. 
It can be computed in \emph{linear-time} in the sample size.
It is \emph{competitive} in that it is guaranteed to perform well
not just for the worst distribution in the class, but for each and
every distribution. 
It \emph{``amplifies''} the data in that it uses just $\Poi(2n)$
samples to approximate the performance of the empirical estimator
with $\Poi(n\sqrt{\log n})$ samples \emph{regardless of the underlining
distribution $p$}, thereby providing an off-the-shelf, distribution-independent, 
``amplification'' of the amount of data available relative to the
estimators used by many practitioners. 
\ignore{
Instead, we consider \emph{competitive optimality} that compares
an estimator's performance to that of the ubiquitous empirical
estimator $f^E$ for \emph{any} distribution.
we derive an estimator $f^*$ that, uniformly over all distributions, 
uses just $2n$ samples to essentially
achieve the performance of $f^E$ with $n\sqrt{\log n}$ samples,
effectively \emph{amplifying} the available data by a factor of $\sqrt{\log n}$.}
As we show in Section~\ref{experiments}, it also works well in
practice, outperforming existing estimator and often working as well
as the empirical estimator with even $n\log n$ samples. 

For a more precise description, let $o(1)$ represent a
quantity that vanishes as
$n\rightarrow\infty$ and write $a\lesssim b$ for $a\leq b(1+o(1))$.
Suppressing small $\epsilon$ for simplicity first, 
we show that 
\[
L_{f^*}(p, 2n)\lesssim L_{f^E}(p,n\sqrt{\log n})+o(1),
\]
where the first right-hand-side term relates the performance of
$f^*$ with $2n$ samples to that of $f^E$ with $n\sqrt{\log n}$
samples. The second term adds a small loss that diminishes
at a rate independent of the support size $k$, and for
fixed $k$ decreases roughly as $1/n$. 
Specifically, we prove,
\begin{theorem}
\label{thm1}
For every property $f$ satisfying the smoothness conditions in
Section~\ref{assum}, there is a constant $C_f$ such that for
all $p\in D_{\cX}$ and all $\epsilon\in{(0,\frac{1}{2})}$,
\[
L_{f^*}(p, 2n)
\le
\Paren{1+\frac{3}{\log^{\epsilon}n}}L_{f^E}(p, n \log^{\frac{1}{2}-\epsilon}n)+C_f\min\left\{ \frac{k}{n}\log^{\epsilon}{n}+\mathcal{\tilde{O}}\Paren{\frac{1}{n}}, \frac{1}{\log^{\epsilon}{n}}  \right\}.
\]
\end{theorem}
The $\tilde{O}$ reflects a multiplicative $\text{polylog}(n)$
factor unrelated to $k$ and $p$.
Again, for normalized support size, $D_{\cX}$ is replaced by $D_k$,
and we also modify $f^*$ as follows: if $k>{n}$, we apply $f^*$, and if $k\leq{n}$, we apply the corresponding min-max estimator~\cite{mmsize}. However, for experiments shown in Section~\ref{experiments}, the original $f^*$ is used without such modification.
In Section~\ref{extensions}, we note that for several properties,
the second term can be strengthened so that it does not depend on $\epsilon$.
\section{Implications}\label{imp}
Theorem~\ref{thm1} has three important implications.
\paragraph{Data amplification}
Many modern applications, such as those arising in genomics and
natural-language processing, concern properties of distributions
whose support size $k$ is comparable to or even larger than the
number of samples $n$.
For these properties, the estimation loss of the empirical estimator
$f^E$ is often much larger than $1/\log^{\epsilon} n$, hence 
the proposed estimator, $f^*$,
yields a much better estimate whose performance parallels that of
$f^E$ with $n\sqrt{\log n}$ samples.
This allows us to amplify the available data by a factor of
$\sqrt{\log n}$ regardless of the underlying distribution. 

Note however that for some properties $f$,
when the underlying distributions are limited to a fixed small support
size, $L_{f^E}(p,n)=\Theta(1/n)\ll1/{\log^{\epsilon} n}$. 
For such small support sizes, $f^*$ may not improve the estimation loss.

\paragraph{Unified estimator}
Recent works either prove efficacy results individually for each
property~\cite{mmentro, mmsize, mml1},
or are not known to be computable in linear time~\cite{mmcover,jnew}.

By contrast, $f^*$ is a linear-time estimator well for all properties satisfying
simple Lipschitz-type and second-order smoothness conditions. 
All properties described earlier: Shannon entropy,
normalized support size, 
normalized suppport coverage, power sum, $L_1$ distance and KL
divergence satisfy these conditions, and $f^*$ therefore applies
to all of them. 

More generally,
recall that a property $f$ is Lipschitz if all $f_x$ are Lipschitz.
It can be shown, e.g.~\cite{learning}, that with $\mathcal{O}(k)$
samples, $f^E$ approximates a $k$-element distribution to a
constant $L_1$ distance, and hence also estimates any Lipschitz property to a
constant loss.
It follows that $f^*$ estimates any Lipschitz property
over a distribution of support size $k$ to constant estimation loss
with $\mathcal{O}(k/\sqrt{\log k})$ samples.
This provides the first  general sublinear-sample estimator
for all Lipschitz properties. 

\paragraph{Competitive optimality}
Previous results were geared towards the estimator's worst estimation
loss over all possible distributions. 
For example, they derived estimators that approximate the
distance to uniformity of any $k$-element distribution with
$\mathcal{O}(k/\log k)$ samples, and showed that this number is optimal as
for some distribution classes estimating this distance requires
$\Omega(k/\log k)$ samples.

However, this approach may be too pessimistic. Distributions
are rarely maximally complex, or are hardest to estimate.
For example, most natural scenes have distinct simple patterns,
such as straight lines, or flat faces, hence can be learned
relatively easily. 

More concretely, consider learning distance to uniformity
for the collection of distributions with entropy bounded by $\log\log k$. 
It can be shown that for sufficiently large $k$, $f^E$ can learn
distance to uniformity to constant estimation loss using
$\mathcal{O}((\log{k})^{\Theta(1)})$ samples. Theorem~\ref{thm1}
therefore shows that the distance to uniformity can be learned
to constant estimation  loss with
$\mathcal{O}((\log{k})^{\Theta(1)}/\sqrt{\log\log k})$ samples. 
(In fact, without even knowing that the entropy is bounded.)
By contrast, the original min-max estimator results would still require 
the much larger $\Omega(k/\log k)$ samples. 

The rest of the paper is organized as follows.
Section~\ref{assum} describes mild smoothness conditions satisfied
by many natural properties, including all those mentioned above. 
Section~\ref{newf} describes the estimator's explicit form and
some intuition behind its construction and performance. 
Section~\ref{extensions} describes two improvements of the estimator
addressed in the supplementary material.
Lastly, Section~\ref{experiments} describes various experiments
that illustrate the estimator's power and competitiveness.
For space considerations, we relegate all the proofs to the
appendix. 
\section{Smooth properties}\label{assum}
Many natural properties, including all those mentioned in the
introduction satisfy some basic smoothness conditions. 
For $h\in(0,1]$, consider the Lipschitz-type parameter
\[
\lipf(h)
\ed
\max_{x}\max_{u, v\in{[0,1]}: \max\{u,v\}\geq h}\frac{|f_x(u) - f_x(v)|}{|u-v|},
\]
and the second-order smoothness parameter, resembling the modulus of
continuity in approximation
theory~\cite{approx,approxconst},\vspace{-.5em}
\[
\omega^2_f(h)
\ed
\max_x\max_{u,v\in{[0,1]}: |u-v|\leq
  {2h}}\left\{\left|\frac{f_x(u)+f_x(v)}{2}-f_x\Paren{\frac{u+v}{2}}\right|\right\}.
\]
We consider properties $f$ satisfying the following conditions:
(1) $\forall x\in \cX$, $f_x(0)=0$;
(2) $\lipf(h)\leq{\text{polylog}(1/h)}$ for $h\in{(0,1]}$;
(3) $\omega^2_f(h)\leq{S_f\cdot h}$ for some absolute constant $S_f$.

Note that the first condition, $f_x(0)=0$, entails no loss of generality.
The second condition implies that $f_x$ is continuous
over $[0,1]$, and in particular right continuous at 0 and
left-continuous at $1$.
It is easy to see that continuity is also essential for consistent estimation. 
Observe also that these conditions are more general than assuming 
that $f_x$ is Lipschitz, as can be seen for entropy where $f_x=x\log
x$, and that all seven properties described earlier satisfy these
three conditions. 
Finally, to ensure that $L_1$ distance satisfies these conditions, we let 
$f_x(p_x)=|p_x-q_x|-q_x$. 
\vspace{-0.25em}
\section{The Estimator \texorpdfstring{$f^*$}{f*} }\label{newf}
\vspace{-0.35em}
Given the sample size $n$, define an \emph{amplification
  parameter} $t>1$, and let $N''\sim \Poi(nt)$ be the amplified sample
size. Generate a sample sequence $X^{N''}$ independently from $p$, and
let $\mul{x}''$ denote the number of times symbol $x$ appeared in $X^{N''}$.
The empirical estimate of $f(p)$ with $\Poi(nt)$ samples is then
\[
f^E(X^{N''})=\sum_{x \in \cX} f_x\left(\frac{\mul{x}''}{N''} \right).
\]
Our objective is to construct an estimator $f^*$ that
approximates $f^E(X^{N''})$ for large $t$ using just $\Poi(2n)$ samples.

Since $N''$ sharply concentrates around $nt$, 
we can show that $f^E(X^{N''})$ can be approximated by
the \emph{modified empirical estimator}, 
\[
f^{\textit{ME}}(X^{N''}) \ed \sum_{x \in \cX} f_x\left(\frac{\mul{x}''}{nt} \right),
\]
where $f_x(p)\ed f_x(1)$ for all $p>1$ and $x\in \cX$.

Since large probabilities are easier to estimate, 
it is natural to set a threshold parameter $s$ and rewrite the
modified estimator as a separate sum over small and large probabilities,
\[
f^{\textit{ME}}(X^{N''}) =
\sum_{x\in \cX} f_x\left(\frac{\mul{x}''}{nt} \right)\indic_{p_x\leq s} + \sum_{x\in \cX} f_x\left(\frac{\mul{x}''}{nt} \right)\indic_{p_x> s}.
\] 
Note however that we do not know the exact probabilities. 
Instead, we draw two independent sample sequences $X^N$ and $X^{N'}$
from $p$, each of an independent $\Poi(n)$ size, 
and let $\mul{x}$ and $\mul{x}'$ be the number of occurrences of
$x$ in the first and second sample sequence respectively.
We then set a \emph{small/large-probability threshold} $s_0$ and
classify a probability $p_x$ as large or small according to
$\mul{x}'$:
\[
f^{\textit{ME}}_S(X^{N''}, X^{N'})\ed
\sum_{x\in \cX} f_x\left(\frac{\mul{x}''}{nt} \right) \indic_{\mul{x}'\leq s_0}
\]
is the \emph{modified small-probability empirical estimator}, and
\[
f^{\textit{ME}}_L(X^{N''}, X^{N'})\ed
\sum_{x\in \cX} f_x\left(\frac{\mul{x}''}{nt} \right) \indic_{\mul{x}'> s_0}
\]
is the \emph{modified large-probability empirical estimator}.
We rewrite the modified empirical estimator as
\[
f^{\textit{ME}}(X^{N''}) = f^{\textit{ME}}_S(X^{N''}, X^{N'}) +f^{\textit{ME}}_L(X^{N''}, X^{N'}).
\]

Correspondingly, we express our estimator $f^*$ as a combination
of small- and large-probability estimators,
\[
f^*(X^N, X^{N'}) \ed f^*_S(X^N, X^{N'})+f^*_L(X^N, X^{N'}).
\]

The \emph{large-probability estimator} approximates $f^{\textit{ME}}_L(X^{N''}, X^{N'})$ as
\[
f^*_L(X^N, X^{N'}) \ed f^{\textit{ME}}_L(X^{N}, X^{N'})
=\sum_{x\in \cX} f_x\left(\frac{\mul{x}}{nt} \right) \indic_{\mul{x}'> s_0}.
\]
Note that we replaced the length-$\Poi(nt)$ sample sequence $X^{N''}$
by the independent length-$\Poi(n)$ sample sequence $X^{N}$.
We can do so as large probabilities are well estimated from
fewer samples. 

The \emph{small-probability estimator} $f^*_S(X^N, X^{N'})$ approximates
$f^{\textit{ME}}_S(X^{N''}, X^{N'})$ and is more involved. We outline its
construction below and details can be found in Appendix~\ref{boundC}. 
The expected value of $f^{\textit{ME}}$ for the small probabilities is
\[
\EE[f^{\textit{ME}}_S(X^{N''}, X^{N'})]
= 
\sum_{x \in \cX}
\EE[\indic_{\mul{x}\le s_0}]\EE\left[f_x\Paren{\frac{\mul{x}''}{nt}}\right].
\]
Let $\lambda_x\ed np_x$ be the expected number of times symbol $x$ 
will be observed in $X^{N}$, and define
\[
g_x(v) \ed f_x \left( \frac{v}{nt}\right) \left(\frac{t}{t-1}\right)^v.
\] 
Then
\[
\EE\left[ f_x\left(\frac{\mul{x}''}{nt} \right)\right]
=
\sum^\infty_{v=0} e^{-\lambda_xt} \frac{(\lambda_xt)^v}{v!} f_x \left( \frac{v}{nt}\right) = e^{-\lambda_x} \sum^\infty_{v=1} e^{-\lambda_x(t-1)} \frac{(\lambda_x(t-1))^v}{v!} g_x \left( v\right).
\]

As explained in Appendix~\ref{Kf}, the sum beyond a
\emph{truncation threshold}
\[
\umax \ed 2s_0t+2s_0-1
\]
is small, hence it suffices to consider the truncated sum
\[
e^{-\lambda_x} \sum^{\umax}_{v=1} e^{-\lambda_x(t-1)}  \frac{(\lambda_x(t-1))^v}{v!} g_x \left( v\right).
\]
Applying the \emph{polynomial smoothing technique} in~\cite{pnas},
Appendix~\ref{boundC} approximates the above summation by
\[
e^{-\lambda_x}\sum^\infty_{v=1} h_{x,v} \lambda_x^v,
\]
where
\[
h_{x,v} =  (t-1)^v \sum^{(\umax\land v)}_{u=1} \frac{g_x(u)(-1)^{v-u}}{(v-u)!u!} \Paren{1-e^{-r}\sum_{j=0}^{v+u}\frac{r^j}{j!}},
\] 
and
\[
r\ed 10s_0t+10s_0.
\]
\ignore{Observe that $1-e^{-r}\sum_{j=0}^{v+u}\frac{r^j}{j!}$
is the tail probability of a $\Poi(r)$ distribution,
hence diminishes rapidly, and attenuates summation terms, beyond $r$,
which can therefore be viewed as a \emph{polynomial smoothing parameter}.}

Observe that $1-e^{-r}\sum_{j=0}^{v+u}\frac{r^j}{j!}$
is the tail probability of a $\Poi(r)$ distribution that
diminishes rapidly beyond $r$.
Hence $r$ determines which summation terms will be attenuated,
and serves as a \emph{smoothing parameter}. 

An unbiased estimator of $e^{-\lambda_x}\sum^\infty_{v=1} h_{x,v} \lambda_x^v$ is
\[
\sum^\infty_{v=1} h_{x,v} v!\cdot \indic_{N_x=v}
=
h_{x,N_x}\cdot N_x!.
\]
Finally, the small-probability estimator is
\[
f^*_S(X^N, X^{N'}) \ed \sum_{x \in \cX} h_{x,N_x} \cdot N_x! \cdot  \indic_{\mul{x}'\leq s_0}.
\]
\section{Extensions}\label{extensions}
In Theorem~\ref{thm1}, for fixed $n$, as $\epsilon\to0$,
the final slack term $1/\log^\epsilon n$ approaches a constant. 
For certain properties it can be improved. 
For normalized support size, normalized support coverage,
and distance to uniformity, a more involved estimator improves this term to
\[
C_{f,\gamma}\min\left\{\frac{k}{n\log^{1-\epsilon}n}+\frac{1}{n^{1-\gamma}}, \frac{1}{\log^{1+\epsilon}{n}}  \right\},
 \]
for any fixed constant $\gamma\in{(0,1/2)}$. 

For Shannon entropy, correcting the bias of
$f^*_L$~\cite{emiller}
and further dividing the probability regions, 
reduces the slack term even more, to
\[
C_{f,\gamma}\min\left\{\frac{k^2}{n^2\log^{2-\epsilon}n}+\frac{1}{n^{1-\gamma}}, \frac{1}{\log^{2+2\epsilon}{n}}\right\}.
\]

Finally, the theorem compares the performance of $f^*$ with $2n$
samples to that of $f^E$ with $n\sqrt{\log n}$ samples. 
As shown in the next section, the performance is often
comparable to that of $n\log n$ samples. 
It would be interesting to prove a competitive result that 
enlarges the amplification to $n\log^{1-\epsilon} n$ 
or even $n\log n$.
This would be essentially the best possible as 
it can be shown that for the symmetric properties mentioned in the introduction,
amplification cannot exceed $\mathcal{O}(n\log n)$. 
\section{Experiments}\label{experiments}
\vspace{-.5em}

We evaluated the new estimator $f^*$ by comparing its performance to
several recent estimators~\cite{mmentro,mmsize,mmcover,pnas,jvhw}.
To ensure robustness of the results, we performed the comparisons for
all the symmetric properties described in the introduction:
entropy, support size, support coverage, power sums, and distance to
uniformity. For each property, we considered six underlying
distributions: uniform, Dirichlet-drawn-, Zipf, binomial, Poisson, and geometric.
The results for the first three properties are shown in
Figures~\ref{figure1}--\ref{figure3}, the plots for the final two 
properties can be found in Appendix~\ref{experimentalresults}.
{\em For nearly all tested properties and distributions, $f^*$ achieved state-of-the-art performance.}

As Theorem~\ref{thm1} implies, for all five properties,
with just $n$ (not even $2n$) samples, $f^*$ performed as well
the empirical estimator $f^E$ with roughly $n\sqrt{\log n}$ samples. 
Interestingly, in most cases $f^*$ performed even better, similar
to $f^E$ with $n\log n$ samples. 

Relative to previous estimators, 
depending on the property and distribution, different previous estimators
were best. 
But in essentially all experiments, $f^*$
was either comparable or outperformed the best previous estimator.
The only exception was PML that attempts to smooth the estimate,
hence performed better on uniform, and near-uniform Dirichlet-drawn
distributions for several properties. 

Two additional advantages of $f^*$ may be worth noting.
First, underscoring its competitive performance for each distribution, 
the more skewed the distribution the better is its relative efficacy. 
This is because most other estimators are optimized for the worst
distribution, and work less well for skewed ones.

Second, by its simple nature, the empirical estimator $f^E$ is
very stable. Designed to emulate $f^E$ for more samples,
$f^*$ is therefore stable as well. 
Note also that $f^E$ is not always the best estimator choice. For
example, it always underestimates the distribution's support size.
Yet even for normalized support size, Figure~\ref{figure2} shows that $f^*$ 
outperforms other estimators including those designed specifically
for this property (except as above for PML on near-uniform distributions).

The next subsection describes the experimental settings. 
Additional details and further interpretation of the observed results 
can be found in Appendix~\ref{experimentalresults}.

\subsection*{Experimental settings}
We tested the five properties on the following distributions:
uniform distribution;
a distribution randomly generated from Dirichlet prior with parameter 2;
Zipf distribution with power $1.5$;
Binomial distribution with success probability $0.3$;
Poisson distribution with mean $3{,}000$;
geometric distribution with success probability $0.99$.

With the exception of normalized support coverage, all other 
properties were tested on distributions of support size $k=10{,}000$.
The Geometric, Poisson, and Zipf distributions were truncated at $k$
and re-normalized. 
The number of samples, $n$, ranged from $1{,}000$ to $100{,}000$,
shown logarithmically on the horizontal axis. 
Each experiment was repeated 100 times and the reported results,
shown on the vertical axis, reflect their mean squared error (MSE).

We compared the estimator's performance with $n$ samples to that
of four other recent estimators as well as 
the empirical estimator with $n$, $n\sqrt{\log n}$, and 
$n\log n$ samples.
We chose the amplification parameter $t$ as $\log^{1-\alpha}n+1$,
where $\alpha\in\{0.0,0.1,0.2,...,0.6\}$ was selected based on independent
data, and similarly for $s_0$. 
Since $f^*$ performed even better than Theorem~\ref{thm1} guarantees,
$\alpha$ ended up between 0 and 0.3 for all properties, 
indicating amplification even beyond $n\sqrt{\log n}$.
The graphs denote
$f^*$ by NEW,
$f^E$ with $n$ samples by Empirical,
$f^E$ with $n\sqrt{\log{n}}$ samples by Empirical+,
$f^E$ with $n\log{n}$ samples by Empirical++,
the pattern maximum likelihood estimator in~\cite{mmcover} by PML,
the Shannon-entropy estimator in~\cite{jvhw} by JVHW,
the normalized-support-size estimator in~\cite{mmsize} and 
the entropy estimator in~\cite{mmentro} by WY,
and the smoothed Good-Toulmin Estimator for normalized support
coverage estimation~\cite{pnas}, slightly modified to account for
previously-observed elements that may appear in the subsequent sample,
by SGT. 

While the empirical and the new estimators have the
same form for all properties, as noted in the introduction, the
recent estimators are property-specific, and each was derived
for a subset of the properties. In the experiments we applied these
estimators to all the properties for which they were derived. 
Also, additional estimators~\cite{ventro, pentro, mentro, gsize, ccover, cacover, jcover} for various properties were compared
in~\cite{mmentro,mmsize,pnas,jvhw} and found to perform
similarly to or worse than recent estimators, hence we do not test
them here.
\ignore{Other estimators~\cite{Jiantao-new} have higher complexity and have
not yet been implemented, and we are not comparing them either.}
\vfill
\begin{figure*}[h]
\begin{multicols}{3}
    \includegraphics[width=0.9\linewidth]{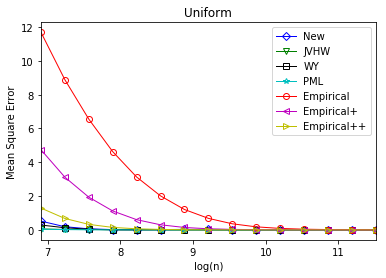}\par 
    \includegraphics[width=0.9\linewidth]{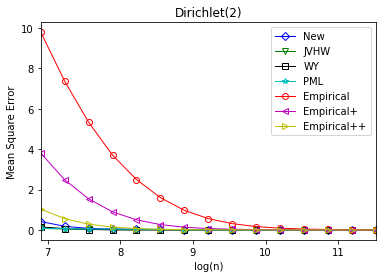}\par 
    \includegraphics[width=0.9\linewidth]{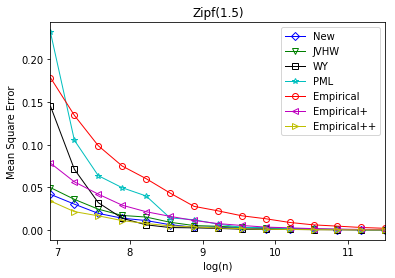}\par
     \end{multicols}
\vspace{-2.5em}
\begin{multicols}{3}
    \includegraphics[width=0.9\linewidth]{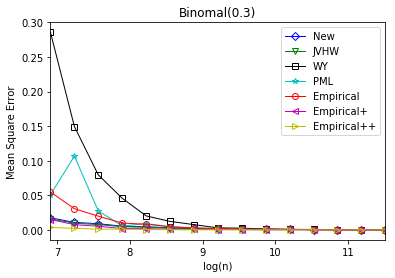}\par
    \includegraphics[width=0.9\linewidth]{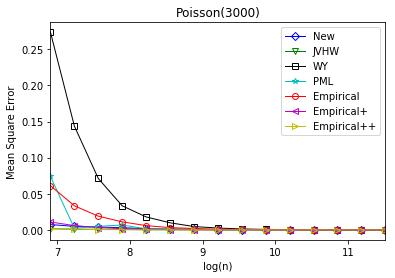}\par
    \includegraphics[width=0.9\linewidth]{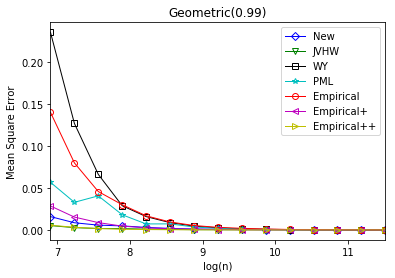}\par
\end{multicols}
\vspace{-1em}
\caption{Shannon Entropy}
\label{figure1}
\end{figure*}

\begin{figure*}[h]
\begin{multicols}{3}
    \includegraphics[width=0.9\linewidth]{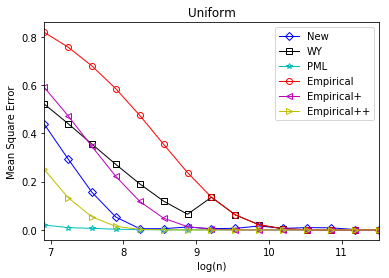}\par 
    \includegraphics[width=0.9\linewidth]{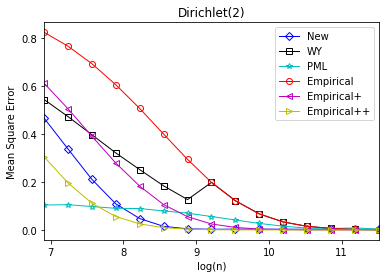}\par 
    \includegraphics[width=0.9\linewidth]{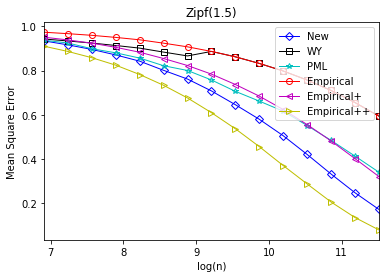}\par
     \end{multicols}
     \vspace{-2.5em}
\begin{multicols}{3}
    \includegraphics[width=0.9\linewidth]{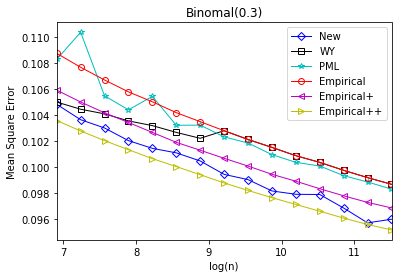}\par
    \includegraphics[width=0.9\linewidth]{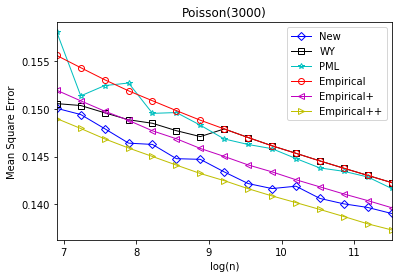}\par
    \includegraphics[width=0.9\linewidth]{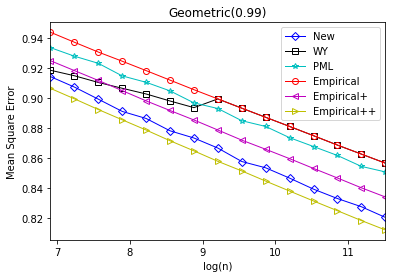}\par
\end{multicols}
\vspace{-1em}
\caption{Normalized Support Size}
\label{figure2}
\end{figure*}
\pagebreak

\begin{figure*}[ht]
\begin{multicols}{3}
    \includegraphics[width=0.9\linewidth]{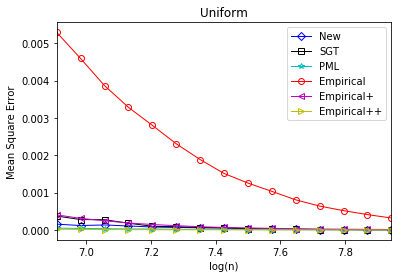}\par 
    \includegraphics[width=0.9\linewidth]{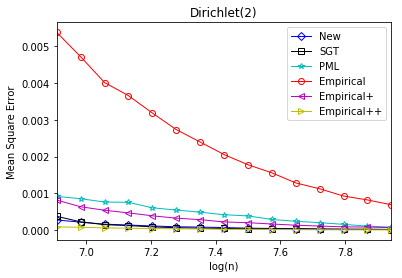}\par 
    \includegraphics[width=0.9\linewidth]{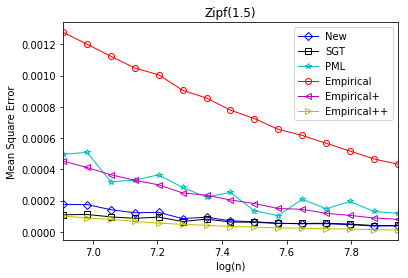}\par
     \end{multicols}
     \vspace{-2.5em}
\begin{multicols}{3}
    \includegraphics[width=0.9\linewidth]{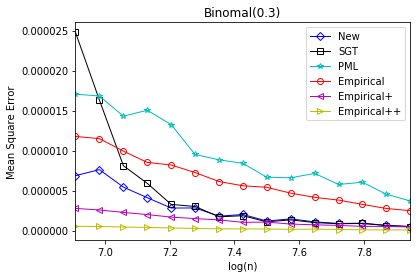}\par
    \includegraphics[width=0.9\linewidth]{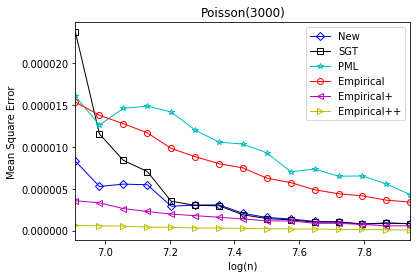}\par
    \includegraphics[width=0.9\linewidth]{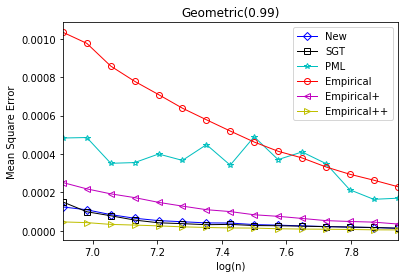}\par
\end{multicols}
\vspace{-1.5em}
\caption{Normalized Support Coverage}
\label{figure3}
\end{figure*}
\vfill
\pagebreak

\ignore{
\begin{table}[h]
  \caption{Comparison of three methods for property estimation} 
  \label{table4}
  \centering
  \begin{tabular}{llll}
    \toprule
    \cmidrule(r){1-2}
    Attributes  & LMM & Estimator ${f}^*$ & PML \\ 
    \midrule
    Permutation invariant& Yes & No & Yes\\ 
    Statistical model & Broad &Broad &Broad\\ 
    Time complexity &Polynomial &Near-linear&Unclear\\
    Functional dependent & No & Yes & No\\
    Parameter tuning & Yes & Yes & No\\
    \bottomrule
  \end{tabular}
\end{table}}

\section{Conclusion}

In this paper, we considered the fundamental learning problem of
estimating properties of discrete distributions. The best-known
distribution-property estimation technique is the ``empirical
estimator'' that takes the data's empirical frequency and plugs it in
the property functional. We designed a general estimator that for a
wide class of properties, uses only $n$ samples to achieve the same
accuracy as the plug-in estimator with $n\sqrt{\log n}$ samples. This
provides an off-the-shelf method for \emph{amplifying} the data available
relative to traditional approaches. For all the properties and
distributions we have tested, the proposed estimator performed as well
as the best estimator(s). A meaningful future research direction would
be to verify the optimality of our results: the amplification factor
$\sqrt{\log n}$ and the slack terms. There are also several important
properties that are not included in our paper, for example, R\'enyi
entropy~\cite{jrenyi} and the generalized distance to uniformity~\cite{yi2018, batu17}. It would
be interesting to determine whether  \emph{data amplification} could be
obtained for these properties as well. 

\appendix
\section{Smooth properties}\label{assumm}
Theorem {\bf $1$} holds for a wide class of properties $f$.
For $h\in(0,1]$, consider the Lipschitz-type parameter
\[
\lipf(h)
\ed
\max_{x}\max_{u, v\in{[0,1]}: \max\{u,v\}\geq h}\frac{|f_x(u) - f_x(v)|}{|u-v|},
\]
and the second-order smoothness parameter, resembling similar approximation-theory
terms~\cite{approx,approxconst},
\[
\omega^2_f(h)
\ed
\max_x\max_{u,v\in{[0,1]}: |u-v|\leq
  {2h}}\left\{\left|\frac{f_x(u)+f_x(v)}{2}-f_x\Paren{\frac{u+v}{2}}\right|\right\}.
\]
We assume that $f$ satisfies the following conditions:
\begin{itemize}
\item
$\forall x\in \cX$, $f_x(0)=0$;
\item
$\lipf(h)\leq{\text{polylog}(1/h)}$ for $h\in{(0,1]}$;
\item
$\omega^2_f(h)\leq{S_f\cdot h}$ for some absolute constant $S_f$.
\end{itemize}

Note that the first condition, $f_x(0)=0$, entails no loss of generality.
The second condition implies that $f_x$ is continuous
over $[0,1]$, and in particular right continuous at 0 and
left-continuous at $1$.
It is easy to see that continuity is also essential for consistent estimation. 
Observe also that these conditions are more general than assuming 
that $f_x$ is Lipschitz, as can be seen for entropy where $f_x=x\log
x$, and that all seven properties described earlier satisfy these
three conditions. 
Finally, to ensure that $L_1$ distance satisfies these conditions, we let 
$f_x(p_x)=|p_x-q_x|-q_x$. Observe also that these conditions are more general than assuming 
that $f_x$ is Lipschitz, as can be seen for entropy where $f_x=x\log x$.

For normalized support size, we modify our estimator $f^*$ as follows: if $k>{n}$, we apply the estimator $f^*$, and if $k\leq{n}$, we apply the corresponding min-max estimator~\cite{mmsize}. However, for experiments shown in Section~\ref{experimentalresults}, the original estimator $f^*$ is used without such modification.

Table~\ref{sample-table2} below summarizes the results on the quantity $\lipf(h)$ and $S_f$ for different properties. Note that for a given property, $\lipf(h)$ is unique while $S_f$ is not.

\begin{table}[h]
  \caption{Values of $\lipf(h)$ and $S_f$ for different properties}
  \label{sample-table2}
  \centering
  \begin{tabular}{llll}
    \toprule
    \cmidrule(r){1-2}
    Property & $f_x(p_x)$ & $\lipf(h)$ & $S_f$\\ 
    \midrule
    KL divergence & $p_x\log \frac{p_x}{q_x}$ & $-\min_{x\in\cX}\log(hq_x)$ & $\log 2$ \\
    $L_1$ distance & $\left|p_x-q_x\right|-q_x$ & 1 & 1 \\ 
    Shannon entropy & $p_x\log \frac{1}{p_x}$ & $-\log(h)$ & $\log 2$ \\ 
    Power sum ($a$)  & $p_x^a$ ($a\geq1$) & 1 & $a$ \\
    Normalized support coverage & $\frac{1-e^{-mp_x}}{m}$ & 1 & 1 \\ 
    Distance to uniformity & $\left|p_x-\frac{1}{k}\right|-\frac{1}{k}$ & 1 & 1 \\ 
    \bottomrule
  \end{tabular}
\end{table}

For simplicity, we denote the \emph{partial expectation} 
$\EE_Y[X]\ed\EE[X\indic_{Y}]$,
and $a\land b \ed \min\{a,b\}$. To simplify our proofs and expressions, we assume that 
the number of samples $n\geq{150}$, 
the amplification parameter $t>2.5$,
and $0<\epsilon\leq{0.1}$.
Without loss of generality, we also assume that $s_0$, $\umax$ and $r$
are integers. 
Finally, set $t= c_1\log^{{1}/{2}-\epsilon} n+1$ and $s_0 = c_2\log^{2\epsilon}n$, where $c_1$ and $c_2$ are fixed constants such that $1\geq{c_1,c_2}>0$ and $c_1\sqrt{c_2}\leq{1/11}$.

\section{Outline}
The rest of the appendix is organized as follows.

In Section~\ref{conc}, we present a few
concentration inequalities for Poisson and Binomial
random variables that will be used in subsequent proofs. 
In Section~\ref{modemp}, we analyze the performance
of the modified empirical estimator $f^{\textit{ME}}$ that estimates 
$p_x$ by $N_x/n$ instead of $N_x/N$. 
We show that $f^{\textit{ME}}$ performs nearly as well as the
original empirical estimator $f^{E}$, but is 
significantly easier to analyze. 

In Section~\ref{largesmall}, we partition the loss of our
estimator, $L_{f^*}(p, nt)$, into three parts:
$\EE[A^2]$,
$\EE[B^2]$, and $\EE[C^2]$,
corresponding to a quantity which is roughly
$L_{f^E}(p, nt) $, the loss incurred by $f^*_L$, and the loss incurred by
$f^*_S$, respectively.

In Section~\ref{boundA}, we bound $\EE[A^2]$
by roughly $L_{f^E}(p,nt)$.
In Section~\ref{boundB}, we bound $\EE[B^2]$: in Section~\ref{biasB} and~\ref{varB}
, we bound the squared bias and variance of $f^*_L$ respectively.

In Section~\ref{Kf}, we partition the series to be estimated in
$\EE[C^2]$ into $R_f$ and $K_f$, and show that it suffices to estimate the 
quantity $K_f$. 
In Section~\ref{constfS}, we outline how we construct
the linear estimator $f^*_S$ based on $K_f$. 
Then, we bound term $\EE[C^2]$: in Section~\ref{varfSec} and~\ref{biasfS}, we bound the variance and squared bias of $f^*_S$ respectively. In Section~\ref{boundCsummary}, we derive a tight bound on $\EE[C^2]$.

In Section~\ref{mainresult}, we prove Theorem {\bf $1$} based on our previous results. 

In Section~\ref{experimentalresults}, we demonstrate the
practical advantages of our methods through experiments on different properties and
distributions. We show that our estimator can even match the
performance of the $n\log n$-sample empirical estimator in estimating various properties.

\section{Preliminary Results}
\subsection{Concentration Inequalities for Poisson and Binomial}\label{conc}
The following lemma gives tight tail probability bounds for Poisson and Binomial random variables.
\begin{Lemma}\label{tailprob}~\cite{concen}
Let $X$ be a Poisson or Binomial random variable with mean $\mu$, then for any $\delta>0$, 
\[
\PP(X\geq{(1+\delta)\mu})\leq{{\left(\frac{e^\delta}{(1+\delta)^{(1+\delta)}}\right)}^{\mu}}\leq{e^{-(\delta^2\land\delta)\mu/3}}
\]
and for any $\delta\in{(0, 1)}$, 
\[
\PP(X\leq{(1-\delta)\mu})\leq{{\left(\frac{e^{-\delta}}{(1-\delta)^{(1-\delta)}}\right)}^{\mu}}\leq{e^{-\delta^2\mu/2}}.
\]
\end{Lemma}
  We have the following corollary by choosing different values of $\delta$.
\begin{Lemma}\label{cortail}
Let $X$ be a Poisson or Binomial random variable with mean $\mu$,
\[
\PP(X\leq{\frac{1}{2}\mu})\leq{e^{-0.15\mu}},\ \ \PP(X\leq{\frac{1}{3}\mu})\leq{e^{-0.30\mu}},
\]
\[
\PP(X\leq{\frac{1}{5}\mu})\leq{e^{-0.478\mu}} \text{, and }\PP(X\leq{\frac{1}{16}\mu})\leq{e^{-0.76\mu}}.
\]
\end{Lemma}

\begin{Lemma}\label{coninv}
Let $N\sim\Poi(n)$,
\begin{align*}
{\EE \left[ \sqrt{\frac{n}{N}}\Bigg\vert N\geq{1}\right]}\leq 1+\frac{3}{n}.
\end{align*}
\end{Lemma}

\begin{proof}
For $N\geq 1$,
\[
\frac{n}{N}\leq \frac{n}{N+1}+\frac{3n}{(N+1)(N+2)},
\]
hence,
\begin{align*}
\EE\left[\frac{n}{N}\Bigg\vert N\geq{1}\right]
&\leq \EE\left[\frac{n}{N+1}\Bigg\vert N\geq{1}\right]+\EE\left[\frac{3n}{(N+1)(N+2)}\Bigg\vert N\geq{1}\right]\\
&\leq \EE\left[\frac{n}{N+1}\right]+\EE\left[\frac{3n}{(N+1)(N+2)}\right]\\
&=\PP[N\geq{1}]+\frac{3}{n}\PP[N\geq{2}]\\
&\leq{1+\frac{3}{n}},
\end{align*}
where the second inequality follows from the fact that $\frac{1}{N+1}$ and $\frac{3n}{(N+1)(N+2)}$ decrease with $N$ 
and the equality follows as $N\sim\Poi(n)$.
\end{proof}

\subsection{The Modified Empirical Estimator}\label{modemp}
The modified empirical estimator
\[
f^{\textit{ME}}(X^N) = \sum_{x \in \cX}f_x\left(\frac{\mul{x}}{n} \right)
\]
estimates the probability of a symbol not by the fraction 
$N_x/N$ of times it appeared, but by $N_x/n$, where $n$ is 
the parameter of the Poisson sampling distribution. 

We show that the original
 and modified empirical estimators have very similar performance. 
\begin{Lemma}\label{empclose}
For all $n\geq{1}$, 
\begin{align*}
\EE\left[\left(f^E(X^N)-f^{\textit{ME}}(X^N)\right)^2\right] \leq \frac{\lipf^2 \left(1/n\right)}{n}.
\end{align*}
\end{Lemma}
\begin{proof}
By the definition of $\lipf(h)$, if $N_x\geq 1$,
\begin{align*}
 & \left|f_x\left(\frac{\mul{x}}{n} \right)-f_x\left(\frac{\mul{x}}{N} \right)\right|\leq\lipf \left( \frac{1}{n}\right)\left| \frac{\mul{x}}{n}-\frac{\mul{x}}{N}\right|=\lipf \left( \frac{1}{n}\right)\frac{\mul{x}}{N}  \frac{|N-n|}{n},
\end{align*}
and if $N_x={0}$,
\begin{align*}
 & \left| f_x\left(\frac{\mul{x}}{n} \right)-f_x\left(\frac{\mul{x}}{N} \right)\right|=0\leq\lipf \left( \frac{1}{n}\right)\frac{\mul{x}}{N}  \frac{|N-n|}{n}.
\end{align*}
Therefore,
\begin{align*}
\EE\left[\left(\sum_{x \in \cX}f_x\left(\frac{\mul{x}}{n} \right)-f_x\left(\frac{\mul{x}}{N} \right)\right)^2\right]
& \leq \EE\left[\left(\sum_{x \in \cX}\lipf \left( \frac{1}{n}\right)\frac{\mul{x}}{N}  \frac{|N-n|}{n}\right)^2\right]\\
& \leq \EE\left[\left(\lipf \left( \frac{1}{n}\right)  \frac{|N-n|}{n}\right)^2\right]\\
& =\frac{\lipf^2 \left(1/n\right)}{n^2} \EE\left[(N-n)^2\right]\\
& =\frac{\lipf^2 \left( 1/n\right)}{n},
\end{align*}
where the last step follows as $N\sim\Poi(n)$ and $\EE\left[(N-n)^2\right]=\Var[N]=n$.
\end{proof}
\section{Large and Small Probabilities}\label{largesmall}
Recall that $f^*$ has the following form
\[
f^*(X^N, X^{N'}) = f^*_S(X^N, X^{N'})+f^*_L(X^N, X^{N'}).
\]
We can rewrite the property as follows
\[
f(p) 
= f(p)-\EE[f^{\textit{ME}}(X^{N''})]+\EE[f^{\textit{ME}}_S(X^{N''}, X^{N'})]+\EE[f^{\textit{ME}}_L(X^{N''}, X^{N'})].
\]
The difference between $f^*(X^N, X^{N'})$ and the actual value $f(p)$
can be partitioned into three terms
\[
f^*(X^N, X^{N'}) - \pf 
=
A+B+C,
\]
where 
\[
A
\ed \EE[f^{\textit{ME}}(X^{N''})-f(p)]
\]
is the bias of the modified empirical estimator with Poi($nt$) samples,
\[
B
\ed
f^*_L(X^N, X^{N'}) - \EE[f^{\textit{ME}}_L(X^{N''}, X^{N'})]
\]
corresponds to the loss incurred by the large-probability estimator $f^*_L$, and
\[
C
\ed f^*_S(X^N, X^{N'}) - \EE[f^{\textit{ME}}_S(X^{N''}, X^{N'})] 
\]
corresponds to the loss incurred by the small-probability estimator $f^*_S$.

By Cauchy-Schwarz inequality, upper bounds on 
$\EE[A^2]$, $\EE[B^2]$, and $\EE[C^2]$, suffice to also upper
bound the estimation loss $L_{f^*}(p, 2n)=\EE[(f^*(X^N,X^{N'})-\pf)^2]$.

In the next section, we bound the squared bias term $\EE[A^2]$. 
In Section~\ref{boundA} and Section~\ref{boundB}, we bound the large- and small-probability terms
$\EE[B^2]$ and $\EE[C^2]$, respectively.

\section{Squared Bias: \texorpdfstring{$\EE[A^2]$}{E[A2]}}\label{boundA}
We relate $\EE[A^2]$ to $L_{f^E}(p,nt)$ through the following inequality.
\begin{Lemma}\label{lemmaA}
Let $T$ be a positive function over ${\mathbb{N}}$,
\[
\EE[A^2] \leq \frac{1+T(n)}{nt}\lipf^2\Paren{\frac{1}{nt}}+\left(1+\frac{1}{T(n)}\right)L_{f^E}(p,nt).
\]
\end{Lemma}
\begin{proof}
We upper bound $\EE[A^2]$ in terms of $L_{f^E}(p,nt)$ using Cauchy-Schwarz inequality and Lemma~\ref{empclose}.
\begin{align*}
\EE[A^2] &= \left(\sum_{x \in \cX}\Paren{\EE\left[f_x\left( \frac{\mul{x}''}{nt}\right)\right] - f_x(p_x)}\right)^2\\
&=\left(\sum_{x \in \cX} \left( \EE\left[f_x\left( \frac{\mul{x}''}{nt}\right)\right] - \EE \left[f_x\left( \frac{\mul{x}''}{N''}\right)\right]\right)+\sum_{x \in \cX} \left( \EE \left[f_x\left( \frac{\mul{x}''}{N''}\right)\right] - f_x(p_x)\right)\right)^2\\
&\leq(1+T(n))\left(\sum_{x \in \cX}  \left(\EE\left[f_x\left( \frac{\mul{x}''}{nt}\right)\right] - \EE \left[f_x\left( \frac{\mul{x}''}{N''}\right)\right]\right)\right)^2+\left(1+\frac{1}{T(n)}\right)L_{f^E}(p,nt)\\
&\leq \frac{1+T(n)}{nt}\lipf^2\Paren{\frac{1}{nt}}+\left(1+\frac{1}{T(n)}\right)L_{f^E}(p,nt).
\end{align*}
\end{proof}
\section{Large Probabilities: \texorpdfstring{$\EE[B^2]$}{E[B2]}}\label{boundB}
Note that 
\begin{align*}
\EE[B^2]
&=\EE[(f^*_L(X^N, X^{N'}) - \EE[f^{\textit{ME}}_L(X^{N''}, X^{N'})])^2]\\
&=Bias(f^*_L)^2+Var(f^*_L),
\end{align*}
where
\[
\text{Bias}(f^*_L)\ed \EE[f^*_L(X^{N}, X^{N'}) - f^{\textit{ME}}_L(X^{N''}, X^{N'})]
\]
and
\[
\text{Var}(f^*_L)\ed \EE[(f^*_L(X^N, X^{N'})-\EE[f^*_L(X^{N}, X^{N'})])^2]
\]
are the bias and variance of $f^*_L(X^N, X^{N'})$ in estimating $\EE[f^{\textit{ME}}_L(X^{N''}, X^{N'})]$, respectively.
We shall upper bound the absolute bias and variance as
\[
|\text{Bias}(f^*_L)| \leq \sqrt{(8S_f) ^2{\Paren{\frac{1}{s_0}\land \frac{k}{n}}}+6\lipf^2\left( \frac{1}{nt}\right)\frac{1}{n}}
\]
and
\[
\Var \left(f^*_L \right) \leq  \lipf^2\Paren{\frac{1}{n}}\frac{4s_0}{n}
\]
in Section~\ref{biasB} and Section~\ref{varB} respectively. 
It follows that 
\begin{Lemma}\label{lemmaB}
For $t>2.5$ and $s_0\geq 1$,
\begin{align*}
  \EE[B^2]= \text{Bias}(f^*_L)^2 + \text{Var}(f^*_L)&\leq  {(8S_f) ^2{\Paren{\frac{1}{s_0}\land \frac{k}{n}}}+10\lipf^2\left( \frac{1}{nt}\right)\frac{s_0}{n}}.
\end{align*}
\end{Lemma}

\subsection{Bounding the Bias of \texorpdfstring{$f^*_L$}{fL}}\label{biasB}
To bound the bias of $f^*_L$, we need the following lemma.
\begin{Lemma}\label{bers}~\cite{ber}
For any binomial random variable $X\sim{B(n, p)}$, continuous function $f_0$, and $p\in[0,1]$, 
\[
\left|\EE \left[{f_0}\left(\frac{X}{n}\right)\right]-{f_0}(p)\right| \leq 3\omega_{f_0}^2\Paren{\sqrt{\frac{p(1-p)}{n}}}.
\]
\end{Lemma}
Recall that $\omega^2_f(h)\leq{S_fh}$ from our assumption. 
\begin{Lemma}\label{berscoro}
For $n\geq{150}$,
\[
\left|\EE_{N\geq{1}} \left[f_x \left(\frac{\mul{x}}{n} \right) - f_x\Paren{p_x} \right] \right|\leq\lipf \left( \frac{1}{n}\right)\frac{p_x}{\sqrt{n}}+3.06S_f\sqrt{\frac{p_x}{n}}.
\]
\end{Lemma}
\begin{proof}
Noting $n\geq{150}$, it follows from Lemma~\ref{coninv} and Lemma~\ref{lemmaB} that
\begin{align*}
\left|\EE_{N\geq{1}} \left[f_x \left(\frac{\mul{x}}{n} \right) - f_x\Paren{p_x} \right] \right|
&\leq \left|\EE_{N\geq{1}} \left[f_x \left(\frac{\mul{x}}{n} \right) -f_x\left(\frac{\mul{x}}{N} \right) \right] \right|+\left|\EE_{N\geq{1}} \left[f_x \left(\frac{\mul{x}}{N} \right) -  f_x\Paren{p_x} \right] \right|\\
&\leq \lipf \left( \frac{1}{n}\right)\frac{p_x}{n}\EE[|N-n|]+{\EE \left[3\omega_{f}^2\Paren{\sqrt{\frac{p_x(1-p_x)}{N}}}\Bigg\vert N\geq{1}\right]}\\
&\leq \lipf \left( \frac{1}{n}\right)\frac{p_x}{n}\sqrt{\EE[(N-n)^2]}+3S_f\sqrt{\frac{p_x}{n}}{\EE \left[ \sqrt{\frac{n}{N}}\Bigg\vert N\geq{1}\right]}\\
&\leq \lipf \left( \frac{1}{n}\right)\frac{p_x}{\sqrt{n}}+3.06S_f\sqrt{\frac{p_x}{n}}.
\end{align*}
\end{proof}
The next lemma essentially bounds the individual bias term for each symbol $x$.
\begin{Lemma}\label{indbias}
For $t>2.5$,
\[
\left|\EE \left[f_x \left(\frac{\mul{x}}{n} \right) -f_x\left( \frac{\mul{x}''}{nt}\right) \right] \right|\leq 5S_f\sqrt{\frac{p_x}{n}}+1.65\lipf\left( \frac{1}{nt}\right)\frac{p_x}{\sqrt{n}}.
\]
\end{Lemma}
\begin{proof}
Using Lemma~\ref{berscoro},
\begin{align*}
&\left|\EE \left[f_x \left(\frac{\mul{x}}{n} \right) -f_x\left( \frac{\mul{x}''}{nt}\right) \right] \right|\\
&\leq \left|\EE_{N, N''\geq{1}} \left[f_x \left(\frac{\mul{x}}{n} \right) -f_x\left( \frac{\mul{x}''}{nt}\right)  \right] \right|+\lipf\left( \frac{1}{n}\right)\EE\left[ \frac{\mul{x}}{n}\right]e^{-n}+\lipf\left( \frac{1}{nt}\right)\EE\left[ \frac{\mul{x}''}{nt}\right]e^{-nt}\\
&\leq \left|\EE_{N\geq{1}} \left[f_x \left(\frac{\mul{x}}{n} \right) -f_x\left( p_x\right) \right] \right|+\left|\EE_{N''\geq{1}} \left[f_x \left(p_x \right) -f_x\left( \frac{\mul{x}''}{nt}\right) \right] \right|+2\lipf\left( \frac{1}{nt}\right)p_xe^{-n}\\
&\leq 5S_f\sqrt{\frac{p_x}{n}}+1.65\lipf\left( \frac{1}{nt}\right)\frac{p_x}{\sqrt{n}},
\end{align*}
where the last step follows from $\lipf\left( \frac{1}{n}\right)\leq\lipf\left( \frac{1}{nt}\right)$, $e^{-n}\leq{\sqrt{n}}$, and $t>2.5$.
\end{proof}
Finally, the next lemma bounds the absolute bias of $f^*_L$.
\begin{Lemma}\label{lemmabiasB}
For $t>2.5$ and $s_0\geq 1$,
\begin{align*}
  \left|Bias(f^*_L)\right|  &\leq  \sqrt{(8S_f) ^2{\Paren{\frac{1}{s_0}\land \frac{k}{n}}}+6\lipf^2\left( \frac{1}{nt}\right)\frac{1}{n}}.
\end{align*}
\end{Lemma}

\begin{proof}
\begin{align*}
\left|Bias(f^*_L)\right|&=\left|\EE\left[\sum_{x \in \cX}f_x\left(\frac{\mul{x}}{n} \right)\indic_{\mul{x}'> s_0} 
-  \sum_{x \in \cX} \EE[\indic_{\mul{x} > s_0}] \EE \left[f_x\left( \frac{\mul{x}''}{nt}\right)\right]\right]\right|\\
& \overset{(a)}{\leq} \sum_{x \in \cX} \EE[\indic_{\mul{x} > s_0}] \left|\EE \left[f_x \left(\frac{\mul{x}}{n} \right) -f_x\left( \frac{\mul{x}''}{nt}\right) \right] \right|\\
& \overset{(b)}{\leq} \sum_{x \in \cX} \EE[\indic_{\mul{x} > s_0}]  \Paren{5S_f\sqrt{\frac{p_x}{n}}+1.65\lipf\left( \frac{1}{nt}\right)\frac{p_x}{\sqrt{n}}} \\
& \overset{(c)}{\leq}  \sqrt{\frac{1}{n}} 5S_f\sum_{x \in \cX} \EE[\indic_{\mul{x} > s_0}]  \sqrt{p_x}+1.65\lipf\left( \frac{1}{nt}\right)\frac{1}{\sqrt{n}}\\
& \overset{(d)}{\leq} {\sqrt{\frac{1}{n}} 5S_f\sqrt{(\sum_{x \in \cX} \EE[\indic_{\mul{x} > s_0}])(\sum_{x \in \cX} \EE[\indic_{\mul{x} > s_0}]p_x)}+1.65\lipf\left( \frac{1}{nt}\right)\frac{1}{\sqrt{n}}}\\
& \overset{(e)}{\leq} 5S_f \sqrt{\frac{1}{s_0}\land \frac{k}{n}}+1.65\lipf\left( \frac{1}{nt}\right)\frac{1}{\sqrt{n}}\\
& \overset{(f)}{\leq} \sqrt{(8S_f) ^2{\Paren{\frac{1}{s_0}\land \frac{k}{n}}}+6\lipf^2\left( \frac{1}{nt}\right)\frac{1}{n}},
\end{align*}
where $(a)$ follows from triangle inequality, $(b)$ follows from Lemma~\ref{indbias}, $(c)$ follows as $\sum_{x \in \cX}p_x=1$ and $\EE[\indic_{\mul{x} > s_0}]\leq{1}$, $(d)$ follows from Cauchy-Schwarz inequality, $(e)$ follows from Markov inequality, i.e., $\EE[\indic_{\mul{x} > s_0}]=\PP[\mul{x} > s_0]\leq{{np_x}/{s_0}}$ and $\sum_{x\in \cX}\EE[\indic_{\mul{x} > s_0}]\leq{k}$, and $(f)$ follows from the inequality $a+b\leq{\sqrt{2(a^2+b^2)}}$.
\end{proof}
\subsection{Bounding the Variance of \texorpdfstring{$f^*_L$}{fL}}\label{varB}
The following lemma exploits independence and bounds the variance of $f^*_L$.
\begin{Lemma}\label{lemmavarB}
For $s_0\geq{1}$, 
\[
\Var \left(f^*_L \right) \leq  \lipf^2\Paren{\frac{1}{n}}\frac{4s_0}{n}.
\]
\end{Lemma}

\begin{proof}
Due to independence,
\begin{align*}
\Var \left(f^*_L \right)
& =\Var \left(\sum_{x \in \cX}f_x\left(\frac{\mul{x}}{n} \right) \indic_{\mul{x}'> s_0} \right)\\
& =\sum_{x \in \cX} \Var{\Paren{f_x \left(\frac{\mul{x}}{n} \right) \indic_{\mul{x}'> s_0}}}\\
& =\sum_{x \in \cX} \Var(\indic_{\mul{x}'> s_0})\EE\left[f_x^2 \left(\frac{\mul{x}}{n} \right)\right]+ \sum_{x \in \cX} \Paren{\EE[\indic_{\mul{x}'> s_0}]}^2\Var{\Paren{f_x \left(\frac{\mul{x}}{n} \right)}}\\
&\leq\sum_{x \in \cX} Var(\indic_{\mul{x}'> s_0})\EE\left[f_x^2 \left(\frac{\mul{x}}{n} \right)\right]+ \sum_{x \in \cX} \Var{\Paren{f_x \left(\frac{\mul{x}}{n} \right)}}.
\end{align*}
To bound the first term,
\begin{align*}
\Var(\indic_{\mul{x}'> s_0})\EE\left[f_x^2 \left(\frac{\mul{x}}{n} \right)\right]
&\leq \Var(\indic_{\mul{x}'> s_0})\EE\left[\lipf^2\Paren{\frac{1}{n}}\Paren{\frac{\mul{x}}{n}}^2\right]\\
&\leq \lipf^2\Paren{\frac{1}{n}}\frac{p_x}{n}\Paren{1+np_x\Var(\indic_{\mul{x}'> s_0})},
\end{align*}
where Lemma~\ref{cortail} further bounds the final term by
\begin{align*}
\Var(\indic_{\mul{x}'> s_0})p_x
&\leq \PP[\mul{x}'\leq s_0]p_x\\
&= e^{-np_x}\sum_{i=0}^{s_0}\frac{(np_x)^{i+1}}{(i+1)!}\frac{i+1}{n}\\
&\leq \frac{s_0+1}{n} e^{-np_x}\sum_{i=0}^{s_0}\frac{(np_x)^{i+1}}{(i+1)!}\\
&= \frac{s_0+1}{n} \PP(1\leq N'_x\leq{s_0+1})\\
&\leq \frac{s_0+1}{n}.
\end{align*}

To bound the second term, let $\hat{N}_x$
be an i.i.d. copy of $N_x$ for each $x$, 
\begin{align*}
2\Var{\Paren{f_x \left(\frac{\mul{x}}{n} \right)}}
&= \Var{\Paren{f_x \left(\frac{\mul{x}}{n} \right) -f_x\left(\frac{\hat{N}_x}{n} \right)}}\\
&= \EE\left[\Paren{f_x \left(\frac{\mul{x}}{n} \right) -f_x\left(\frac{\hat{N}_x}{n} \right)}^2\right]\\
&\leq \EE\left[\lipf^2\Paren{\frac{1}{n}}\Paren{\frac{\mul{x}}{n} - \frac{\hat{N}_x}{n} }^2\right]\\
&=2\lipf^2\Paren{\frac{1}{n}}\frac{p_x}{n}.
\end{align*}
A simple combination of these bounds yields the lemma.
\end{proof}

\section{Small Probabilities: \texorpdfstring{$\EE[C^2]$}{E[C2]}}\label{boundC}
As outlined in Section~\ref{newf}, the quantity to be estimated in $C$ is
\[
\EE[f^{\textit{ME}}_S(X^{N''}, X^{N'})] 
= \!\!\sum_{x \in \cX}  \EE[\indic_{\mul{x} \leq s_0}] \EE\left[f_x\left(\frac{\mul{x}''}{nt} \right)\right]  
\!\!=\!\!\sum_{x \in \cX} \EE[\indic_{\mul{x} \leq s_0}]  \sum^\infty_{v=1} e^{-\lambda_xt} \frac{(\lambda_xt)^v}{v!} f_x \left( \frac{v}{nt}\right).
\]
We truncate the inner summation according to the threshold $\umax=2s_0t+2s_0-1$ and define
\[
 K_f \ed \sum_{x \in \cX} \EE[\indic_{\mul{x} \leq s_0}] \sum^{\umax}_{v=1}  e^{-\lambda_xt} \frac{(\lambda_xt)^v}{v!} f_x \left( \frac{v}{nt}\right)
\]
and
\[
R_f \ed  \sum_{x \in \cX}  \EE[\indic_{\mul{x} \leq s_0}] \sum^{\infty}_{v=\umax+1}  e^{-\lambda_xt} \frac{(\lambda_xt)^v}{v!} f_x \left( \frac{v}{nt}\right),
\]
then,
\[
\EE[f^{\textit{ME}}_S(X^{N''}, X^{N'})]  =  K_f + R_f.
\]
The truncation threshold $\umax$ is calibrated such that for each symbol $x$, 
\[
\sum^{\umax}_{v=1}  e^{-\lambda_xt} \frac{(\lambda_xt)^v}{v!} f_x \left( \frac{v}{nt}\right)
\]
contains only roughly $\log(n)$ terms and $R_f^2$ is sufficiently small and contributes only to the slack term in Theorem {\bf $1$}, as shown in Lemma~\ref{lemmaKf}.
In Section~\ref{constfS}, we shall thus construct $f^*_S(X^N, X^{N'})$ to estimate $K_f$ instead of $\EE[f^{\textit{ME}}_S(X^{N''}, X^{N'})]$. 

Analogous to Section~\ref{boundB}, define 
\[
\text{Bias}(f^*_S)\ed \EE[f^*_S(X^{N}, X^{N'}) - K_f]
\]
and
\[
\text{Var}(f^*_S)\ed \EE[(f^*_S(X^N, X^{N'})-\EE[f^*_S(X^{N}, X^{N'})])^2]
\]
as the bias and variance of $f^*_S(X^N, X^{N'})$ in estimating $K_f$, respectively, it follows that
\begin{align*}
\EE[C^2]
&=\EE[(f^*_S(X^N, X^{N'}) - \EE[f^{\textit{ME}}_L(X^{N''}, X^{N'})])^2]\\
&=\EE\left[\Paren{f^*_S(X^N, X^{N'})-(K_f+R_f)}^2\right]\\
&= \Var(f^*_S)+\Paren{\text{Bias}(f^*_S)-R_f}^2\\
&\leq \Var(f^*_S)+\Paren{1+{\log n}}\Paren{\text{Bias}(f^*_S)}^2+\Paren{1+\frac{1}{\log n}}R_f^2.
\end{align*}

We shall upper bound the variance and squared bias as
\[
\Var(f^*_S) \leq {\Paren{n\land k}} \Paren{\lipf \left( \frac{1}{nt}\right) \frac{\umax}{nt}}^2 e^{4r(t-1)}. 
\]
and
\[
\text{Bias}(f^*_S)^2\leq {\Paren{1\land \frac{k^2}{n^2}}}  e^{-4s_0t} \lipf^2 \left( \frac{1}{nt}\right)
\]
in Section~\ref{varfSec} and Section~\ref{biasfS} respectively. 
It follows by simple algebraic manipulation that 
\begin{Lemma}\label{lemmaC}
For the set of parameters specified in Section~\ref{assum}, if $c_1\sqrt{c_2}\leq{1/11}$, $t>{2.5}$, $n\geq150$, and $1 \leq s_0 \leq \log^{0.2}(n)$,
\begin{align*}
 \EE[C^2]  &\leq 13^2{\Paren{1\land \frac{k}{n}}} \lipf^2 \left( \frac{1}{nt}\right)\Paren{\frac{\log^2 n}{e^{0.6s_0}}}.
\end{align*}
\end{Lemma}
\subsection{Bounding the Last Few Terms}\label{Kf}
We now show that $R_f^2$ is sufficiently small and only contributes to the slack term in Theorem {\bf $1$}. The key is to divide the sum into two parts and apply Lemma~\ref{cortail} seperately.

\begin{Lemma}\label{lemmaKf}
For $n\geq{150}$, $1\leq s_0\leq \log^{0.2}n$, and $t>{2.5}$,
\[
R_f^2 \leq {\Paren{7.1{\Paren{1\land \frac{k}{n}}}  \lipf \left( \frac{1}{n}\right)  e^{-0.3s_0} \log(n)}}^2+\Paren{\frac{7.1}{n^{3.8}}\lipf \left( \frac{1}{n}\right)}^2.
\]
\end{Lemma}
\begin{proof}
Recall that $\umax=2s_0t+2s_0$, we upper bound the absolute value of $R_f$ as 
\begin{align*}
|R_f| &=  \left|\sum_{x \in \cX}\sum^{s_0}_{u=0} \sum^\infty_{v=2s_0t+2s_0} e^{-\lambda_x} \frac{\lambda_x^u}{u!} e^{-\lambda_xt} \frac{(\lambda_xt)^v}{v!} f_x \left( \frac{v}{nt}\right)\right|\\
& \leq \sum_{x \in \cX} \sum^\infty_{u+v=2s_0t+2s_0} e^{-\lambda_x(t+1)} \frac{(\lambda_x(t+1))^{u+v}}{(u+v)!} \cdot\\
&\Paren{\lipf \left( \frac{2s_0t+2s_0}{nt}\right)   \frac{u+v}{nt} } \sum^{s_0}_{u=0} {\binom{u+v}{u}} \left(\frac{1}{t+1} \right)^{u} \left(\frac{t}{t+1} \right)^{v}\\
& = \sum_{x \in \cX} \sum^\infty_{u+v=2s_0t+2s_0} e^{-\lambda_x(t+1)} \frac{(\lambda_x(t+1))^{u+v}}{(u+v)!} \cdot\\
& \Paren{\lipf \left( \frac{2s_0t+2s_0}{nt}\right)   \frac{u+v}{nt} }  \PP\Paren{B\left(u+v,\frac{1}{t+1}\right)\leq{s_0}}\\
& \leq \lipf \left( \frac{1}{n}\right)\sum_{x \in \cX} \sum^\infty_{u+v=2s_0t+2s_0} e^{-\lambda_x(t+1)} \frac{(\lambda_x(t+1))^{u+v}}{(u+v)!}  \frac{u+v}{nt} \PP\Paren{B\left(u+v,\frac{1}{t+1}\right)\leq{s_0}}.
\end{align*}
For $u+v\geq{2s_0t+2s_0}$, Lemma~\ref{cortail}  yields
\begin{align*}
\PP\left(B\left(u+v,\frac{1}{t+1}\right)\leq{s_0}\right)\leq{e^{-0.15(u+v)/(t+1)}}\leq{e^{-0.3s_0}}.
\end{align*}
Truncate the inner summation at $u+v={5(t+1)\log n}$ and apply the above inequality,
\begin{align*}
& \sum_{x \in \cX} \sum^{5(t+1)\log n}_{u+v=2s_0t+2s_0} e^{-\lambda_x(t+1)} \frac{(\lambda_x(t+1))^{u+v}}{(u+v)!}  \frac{u+v}{nt}\PP\Paren{B\left(u+v,\frac{1}{t+1}\right)\leq{s_0}}\\
& \leq   \frac{5(t+1)\log n}{nt} e^{-0.3s_0} \sum_{x \in \cX} \sum^{5(t+1)\log n}_{u+v=2s_0t+2s_0} e^{-\lambda_x(t+1)} \frac{(\lambda_x(t+1))^{u+v}}{(u+v)!}\\
& \leq   \frac{5(t+1)\log n}{nt} e^{-0.3s_0} \sum_{x \in \cX} \PP\Paren{\Poi(\lambda_x(t+1))\geq{2s_0t+2s_0}}\\
& \leq   \frac{5(t+1)\log n}{nt} e^{-0.3s_0} \sum_{x \in \cX}\Paren{1 \land \lambda_x}\\
&  \leq 7{\Paren{1\land \frac{k}{n}}} e^{-0.3s_0}\log n  ,
\end{align*}
where the second last inequality follows from the Markov's inequality and the last one follows from $\sum_{x\in\cX}\lambda_x=n$ and $|\cX|=k$.

For $u+v\geq{5(t+1)\log n}+1$, Lemma~\ref{cortail}, $1\leq s_0\leq \log^{0.2}n$, and $n\geq{150}$ together yield
\[
\frac{u+v}{t+1}\geq{5\log n}\geq{16\log^{0.2}n}\geq{16s_0}
\]
and
\[
\PP\left(B\left(u+v,\frac{1}{t+1}\right)\leq{s_0}\right)\leq{e^{-0.76\times5\log n}}\leq\frac{1}{n^{3.8}}.
\]
It remains to consider the following partial sum.
\begin{align*}
& \sum_{x \in \cX} \sum^{\infty}_{u+v=5(t+1)\log n+1} e^{-\lambda_x(t+1)} \frac{(\lambda_x(t+1))^{u+v}}{(u+v)!}\frac{u+v}{nt}  \PP\Paren{B\left(u+v,\frac{1}{t+1}\right)\leq{s_0}}\\
&  \leq \frac{1}{n^{3.8}} \frac{1}{nt}\sum_{x \in \cX} \sum^{\infty}_{u+v=5(t+1)\log n+1} e^{-\lambda_x(t+1)} \frac{(\lambda_x(t+1))^{u+v}}{(u+v)!} (u+v) \\
&  \leq \frac{1}{n^{3.8}} \frac{1}{nt}\sum_{x \in \cX} \lambda_x(t+1) \\
&  \leq \frac{1.4}{n^{3.8}},
\end{align*}
where the last inequality comes from $\sum_{x\in\cX}\lambda_x=n$ and $t>2.5$. The lemma follows from Cauchy-Schwarz inequality.
\end{proof}

\subsection{Estimator Construction for Small Probabilities: \texorpdfstring{$f^*_S$}{fS}}\label{constfS}
 According to Lemma~\ref{lemmaKf}, it suffices to estimate 
\[
 K_f = \sum_{x \in \cX} \EE[\indic_{\mul{x} \leq s_0}] \sum^{\umax}_{u=1}  e^{-\lambda_xt} \frac{(\lambda_xt)^u}{u!} f_x \left( \frac{u}{nt}\right).
\]
Recall that
\[
g_x(u) = f_x \left( \frac{u}{nt}\right) \left(\frac{t}{t-1}\right)^u,
\] 
we can rewrite $K_f$ as
 \[
{  K_f=\sum_{x \in \cX} \EE[\indic_{\mul{x} \leq s_0}] e^{-\lambda_x} \sum^{\umax}_{u=1} e^{-\lambda_x (t-1)}\frac{(\lambda_x (t-1))^u}{u!}  g_x(u)}.
\]
Let 
\[
f_u(y) \ed J_{2u}(2\sqrt{y}) = \sum^\infty_{i=0} \frac{(-1)^i y^{i+u}}{i! (i+2u)!},
\]
where $J_{2u}$ is the Bessel function of the first kind with parameter $2u$. Our estimator is motivated by the following equality.
\begin{Lemma}\label{keyequal}
For any $u\in\mathbb{Z}^+$ and $y\geq 0$,
\[
\int^\infty_{0} e^{-\alpha} \alpha^u f_u(\alpha y) d\alpha  = e^{-y} y^u.
\]
\end{Lemma}
\begin{proof}
By Fubini's theorem and the series expansion of $f_u$,
\begin{align*}
\int^\infty_{0} e^{-\alpha} \alpha^uf_u(\alpha y) d\alpha
& = \int^\infty_{0} e^{-\alpha} \alpha^u \sum^\infty_{i=0} \frac{(-1)^i (\alpha y)^{i+u}}{(i!)(i+2u)!}  d\alpha \\
& =   \sum^\infty_{i=0} \frac{(-1)^i ( y)^{i+u}}{(i!)(i+2u)!}    \int^\infty_{0} e^{-\alpha}  \alpha^{i+2u} d\alpha.
\end{align*}

Observe that the integral is actually $\Gamma(i+2u+1)$ and equals to $(i+2u)!$, 
\begin{align*}
\sum^\infty_{i=0} \frac{(-1)^i ( y)^{i+u}}{(i!)(i+2u)!}    \int^\infty_{0} e^{-\alpha}  \alpha^{i+2u} d\alpha 
& =   \sum^\infty_{i=0} \frac{(-1)^i ( y)^{i+u}}{(i!)(i+2u)!}   (i+2u)! \\
& =    \sum^\infty_{i=0} \frac{(-1)^i ( y)^{i+u}}{i!}  \\
& = e^{-y} y^u.
\end{align*}
\end{proof}

Therefore, let
\[
h_{x}(\lambda_x) \ed e^{-\lambda_x}\sum^{\umax}_{u=1} \frac{g_x(u)}{u!}\left( \int^{\infty}_{0}  e^{-\alpha} \alpha^u f_u(\alpha \lambda_x(t-1)) d\alpha \right),
\] 
we can rewrite
\[
K_f = \sum_{x \in \cX} \EE[\indic_{\mul{x} \leq s_0}] h_x(\lambda_x).
\]
We apply the \emph{polynomial smoothing technique} in~\cite{pnas} and approximate $h_x(y)$ by
\[
\hat{h}_{x}(\lambda_x) \ed e^{-\lambda_x} \sum^{\umax}_{u=1} \frac{g_x(u)}{u!}\left( \int^r_{0}  e^{-\alpha} \alpha^u f_u(\alpha \lambda_x (t-1)) d\alpha \right),
\]
where $r$ is the polynomial smoothing parameter defined in Section~\ref{newf}.

We now expand $\hat{h}_{x}(\lambda_x)$ as a product of $e^{-\lambda_x}$ and a power series of $\lambda_x$.
\begin{Lemma}\label{expandh}
For $t>2.5$,
\[
\hat{h}_{x}(\lambda_x) = e^{-\lambda_x}\sum^\infty_{v=1} h_{x,v} \lambda_x^v,
\]
where
\[
h_{x,v} =  (t-1)^v \sum^{(\umax\land v)}_{u=1} \frac{g_x(u)(-1)^{v-u}}{(v-u)!u!} \Paren{1-e^{-r}\sum_{j=0}^{v+u}\frac{r^j}{j!}}.
\]
\end{Lemma}
\begin{proof}
By Fubini's theorem and the series expansion of $f_u$,
\begin{align*}
\int^r_{0}  e^{-\alpha} \alpha^u f_u(\alpha \lambda_x(t-1)) d\alpha 
& =  \int^r_0 e^{-\alpha} \alpha^u \sum^\infty_{i=0} \frac{(-1)^i (\alpha \lambda_x(t-1))^{i+u}}{(i!)(i+2u)!} d \alpha \\
& =  \sum^\infty_{i=0} \frac{(-1)^i ( \lambda_x(t-1))^{i+u}}{(i!)(i+2u)!}  \int^r_0 e^{-\alpha} \alpha^{i+2u} d \alpha \\
& = \sum^\infty_{i=0} \frac{(-1)^i ( \lambda_x(t-1))^{i+u}}{i!} \left(
1 - e^{-r} \sum^{i+2u}_{j=0} \frac{r^j}{j!} \right).
\end{align*}
Hence, 
\begin{align*}
\hat{h}_{x}(\lambda_x) 
& = e^{-\lambda_x} \sum^{\umax}_{u=1} \frac{g_x(u)}{u!}\left( \int^r_{0}  e^{-\alpha} \alpha^u f_u(\alpha \lambda_x(t-1)) d\alpha \right)\\
& = e^{-\lambda_x} \sum^{\umax}_{u=1} \frac{g_x(u)}{u!}\sum^\infty_{i=0} \frac{(-1)^i ( \lambda_x(t-1))^{i+u}}{i!} \left(1 - e^{-r} \sum^{i+2u}_{j=0} \frac{r^j}{j!} \right)\\
& = e^{-\lambda_x} \sum^\infty_{v=1}\left[ (t-1)^v \sum^{(\umax\land v)}_{u=1} \frac{g_x(u)(-1)^{v-u}}{(v-u)!u!} \Paren{1-e^{-r}\sum_{j=0}^{v+u}\frac{r^j}{j!}}\right]\lambda_x^v\\
& = e^{-\lambda_x} \sum^\infty_{v=1} h_{x,v} \lambda_x^v
\end{align*}
\end{proof}

An unbiased estimator of $\hat{h}_{x}(\lambda_x) = e^{-\lambda_x}\sum^\infty_{v=1} h_{x,v} \lambda_x^v$ is
\[
\sum^\infty_{v=1} h_{x,v} v!\cdot \indic_{N_x=v}
=
h_{x,N_x}\cdot N_x!.
\]
Our small-probability estimator is thus
\[
f^*_S(X^N, X^{N'}) =\sum_{x \in \cX} h_{x,N_x} \cdot N_x!  \cdot \indic_{\mul{x}'\leq s_0}.
\]
In the next section, we show that the connection between $h_x(\lambda)$ and $\hat{h}_x(\lambda)$ leads to a small expected squared loss of $f^*_S$.
 
\subsection{Bounding the Variance of \texorpdfstring{$f^*_S$}{fS}}\label{varfSec}
First we upper bound the variance of $f^*_S$ in terms of the coefficients $h_{x,v}$.
 
\begin{Lemma}\label{varfS}
\label{lem:general_bounds}
The variance of $f^*_S$ is bounded by
\[
\Var(f^*_S) \leq (n\land k)  \max_{x\in\cX}\max_{v} h^2_{x, v} v!^{2}.
\]
\end{Lemma}
\begin{proof}
First observe that independence and $\Var[X]\leq\EE[X^2]$ imply
\begin{align*}
\Var(f^*_S)&=\Var(\sum_{x\in \cX}\sum^\infty_{v=1}  h_{x,v} v!  \indic_{N_x=v} \indic_{\mul{x}'\leq s_0})\\
&=\sum_{x\in \cX}\Var(\sum^\infty_{v=1}  h_{x,v} v!  \indic_{N_x=v} \indic_{\mul{x}'\leq s_0})\\
&\leq \sum_{x\in \cX} \EE[(\sum^\infty_{v=1}  h_{x,v} v!  \indic_{N_x=v} \indic_{\mul{x}'\leq s_0})^2].
\end{align*}
Note that $\indic_{N_x=u}\indic_{N_x=v}=0$ for any $u\not=v$, we can rewrite the last summation as 
\begin{align*}
\sum_{x\in \cX} \EE[\sum^\infty_{v=1}  (h_{x,v} v!)^2 \indic_{N_x=v} \indic_{\mul{x}'\leq s_0}]
&\leq \max_{x\in\cX}\max_{v} h^2_{x,v} v!^{2} \EE[\sum_{x\in \cX}\sum^\infty_{v=1} \indic_{N_x=v} \indic_{\mul{x}'\leq s_0}]\\
&\leq \max_{x\in\cX}\max_{v} h^2_{x,v} v!^{2} \EE[\sum_{x\in \cX}\sum^\infty_{v=1} \indic_{N_x=v} ]\\
&\leq (n\land k)\max_{x\in\cX}\max_{v} h^2_{x,v} v!^{2},
\end{align*}
where the last inequality follows from $\sum_{x\in \cX}\sum^\infty_{v=1} \indic_{N_x=v}\leq{N\land k}$ and $\EE[N]=n$.
\end{proof}

The following lemma provides a uniform bound on $|h_{x,v} v!|$, which, by Lemma~\ref{varfS}, is sufficient to bound the variance of $f^*_S$.
\begin{Lemma}\label{boundcoeff}
For $t>2.5$,
\[
|h_{x,v} v!| \leq \lipf \left( \frac{1}{nt}\right) \frac{\umax}{nt} e^{2r(t-1)}.
\]
\end{Lemma}
\begin{proof}
From the definition of $g_x(u)$,
\begin{align*}
|h_{x,v} v!|
& \leq (t-1)^v e^{-r}\sum^{(\umax\land v)}_{u=1} \frac{|g_x(u)|v!}{(v-u)!u!} \sum^\infty_{j=v+u+1}
\frac{r^j}{j!} \\
& = e^{-r}\sum^{(\umax\land v)}_{u=1} \left|f_x\left(\frac{u}{nt}\right)\right| t^u (t-1)^{v-u}{\binom{v}{u}} \sum^\infty_{j=v+u+1}
\frac{r^j}{j!} \\
& \leq \lipf \left( \frac{1}{nt}\right) \frac{\umax}{nt} e^{-r}\sum^{(\umax\land v)}_{u=1}  t^u (t-1)^{v-u}{\binom{v}{u}} \sum^\infty_{j=v+u+1}
\frac{r^j}{j!} \\
& \leq \lipf \left( \frac{1}{nt}\right) \frac{\umax}{nt} e^{-r}\sum^\infty_{j=v+2}
\frac{r^j}{j!} \sum^{(\umax\land v)}_{u=1}  {\binom{v}{u}} t^u (t-1)^{v-u}.
\end{align*}
For $t>2.5$, the binomial expansion theorem yields
\begin{align*}
\sum^{(\umax\land v)}_{u=1}  {\binom{v}{u}} t^u (t-1)^{v-u} &\leq (2t-1)^v. 
\end{align*}
Combining the above inequality with the previous upper bound,
\begin{align*}
|h_{x,v} v!|
& \leq \lipf \left( \frac{1}{nt}\right) \frac{\umax}{nt} e^{-r}\sum^\infty_{j=v+2}
\frac{r^j}{j!} (2t-1)^v  \\
& \leq \lipf \left( \frac{1}{nt}\right) \frac{\umax}{nt} e^{-r}\sum^\infty_{j=v+2}
\frac{((2t-1)r)^j}{j!}  \\
& \leq \lipf \left( \frac{1}{nt}\right) \frac{\umax}{nt} e^{-r}\sum^\infty_{j=0}
\frac{((2t-1)r)^j}{j!}  \\
& = \lipf \left( \frac{1}{nt}\right) \frac{\umax}{nt} e^{2r(t-1)},
\end{align*}
where the last equality follows from the Taylor expansion of $e^y$.
\end{proof}

The above results yield the following upper bound on $\Var(f^*_S)$.
\begin{Lemma}
For the set of parameters specified in Section~\ref{assum}, if $c_1\sqrt{c_2}\leq{1/11}$ and $t>{2.5}$, then
\begin{align*}
\Var(f^*_S) &\leq {\Paren{1\land \frac{k}{n}}}  \frac{9s^2_0}{n^{0.22}}\lipf^2 \left( \frac{1}{nt}\right).
\end{align*}
\end{Lemma}
\begin{proof}
By Lemma~\ref{varfS} and Lemma~\ref{boundcoeff}, 
\[
\Var(f^*_S) \leq (n\land k)\Paren{\lipf \left( \frac{1}{nt}\right) \frac{\umax}{nt}}^2 e^{4r(t-1)}.
\]
Note that $t>2.5$, 
\[
\frac{\umax}{nt}=\frac{2s_0t+2s_0-1}{nt}\leq \frac{2s_0t+2s_0}{nt}\leq{\frac{3s_0}{n}},
\]
and since $c_1\sqrt{c_2}\leq{0.1}$,
\[
4r(t-1)=40s_0(t+1)(t-1)\leq 94s_0(t-1)^2=94c_1^2c_2\log n\leq 0.78\log n.
\]
Hence, 
\[
\Paren{\lipf \left( \frac{1}{nt}\right) \frac{\umax}{nt}}^2 e^{4r(t-1)}\leq \Paren{\frac{3s_0}{n}}^2 \lipf^2 \left( \frac{1}{nt}\right) n^{0.78}\leq \frac{1}{n} \frac{9s^2_0}{n^{0.22}} \lipf^2 \left( \frac{1}{nt}\right),
\]
which implies that
\[
\Var(f^*_S)\leq{\Paren{1\land \frac{k}{n}}}  \frac{9s^2_0}{n^{0.22}} \lipf^2 \left( \frac{1}{nt}\right).
\]
\end{proof}

\subsection{Bounding the Bias of \texorpdfstring{$f^*_S$}{fS}}\label{biasfS}
Recall that 
\begin{align*}
\text{Bias}(f^*_S) 
&= \EE[f^*_S(X^{N}, X^{N'}) - K_f]\\
&= \EE[\sum_{x \in \cX} h_{x,N_x} \cdot N_x!  \cdot \indic_{\mul{x}'\leq s_0}-\sum_{x \in \cX}  h_x(\lambda_x)\EE[\indic_{\mul{x} \leq s_0}]]\\
&= \sum_{x \in \cX} (\hat{h}_x(\lambda_x)-h_x(\lambda_x))\EE[\indic_{\mul{x} \leq s_0}],
\end{align*}
which yields
\begin{align*}
\left|\text{Bias}(f^*_S)\right| 
&\leq \sum_{x \in \cX} \left|\hat{h}_x(\lambda_x)-h_x(\lambda_x)\right|\\
&= \sum_{x \in \cX} \left|\sum^{\umax}_{u=1} \frac{g_x(u)}{u!}\left( \int^{\infty}_{r}  e^{-\alpha} \alpha^u f_u(\alpha \lambda_x(t-1)) d\alpha \right)\right|
\end{align*}

The following lemma bounds $|f_u(y)|$ by simple functions and allows us to deal with the integral.
\begin{Lemma}\label{boundbessel}
For $u\geq{1}$ and $y\geq{0}$,
\[
|f_u(y)|\leq{1 \land \frac{y}{u+1}}.
\]
\end{Lemma}
\begin{proof}
For $u\geq{1}$ and $y\geq{0}$, we have the following well-known upper bound~\cite{bessel} for the Bessel function of the first kind.
\begin{align*}
J_u(y)\leq{1\land \frac{(y/2)^u}{u!}},
\end{align*}
which implies
\[
f_u(y) = J_{2u}(2\sqrt{y}) \leq{1\land \frac{(y)^u}{(2u)!}}.
\]
If $y\geq{u+1}$, then
\[
f_u(y) \leq{1 \land \frac{(y)^u}{(2u)!}}\leq{1}\leq{\frac{y}{u+1}}.
\]
If ${u+1}>{y}\geq{0}$, then
\[
f_u(y) \leq{1 \land \frac{(y)^u}{(2u)!}}\leq{\frac{(y)^{u}}{(2u)!}}\leq{\frac{(u+1)^{u}}{(2u)!}\frac{y}{u+1}}\leq{\frac{y}{u+1}}\leq{1}.
\]
\end{proof}

To bound $\left|\text{Bias}(f^*_S)\right|$, it suffices to bound $\lvert \hat{h}_x(\lambda_x)-h_x(\lambda_x) \rvert$.
The lemma below follows from the first half of Lemma~\ref{boundbessel}, i.e., $|f_u(y)|\leq{{y}/{(u+1)}}$.
\begin{Lemma}\label{boundbessel1}
For $t>2.5$ and $s_0\geq 1$,
\begin{align*}
\lvert \hat{h}_x(\lambda_x)-h_x(\lambda_x) \rvert&\leq \frac{\lambda_x}{n}\lipf \left( \frac{1}{nt}\right)e^{-2s_0t}.
\end{align*}
\end{Lemma}
\begin{proof}
Since $|f_u(y)|\leq{{y}/{(u+1)}}$, 
\begin{align*}
\lvert \hat{h}_x(\lambda_x)-h_x(\lambda_x)\rvert
& \leq  \sum^{\umax}_{u=1} \frac{|g_x(u)|}{(u+1)!}y(t-1) \int^{\infty}_{r}  e^{-\alpha} \alpha^{u+1} d\alpha.
\end{align*}
Note that the integral is actually the incomplete Gamma function, we can rewrite the last term as
\begin{align*}
 \lambda_x(t-1)\sum^{\umax}_{u=1} \frac{|g_x(u)|}{(u+1)!} (u+1)! e^{-r} \sum^{u+1}_{i=0} \frac{r^i}{i!}
& = \lambda_x(t-1)\sum^{\umax}_{u=1}{|g_x(u)|} e^{-r} \sum^{u+1}_{i=0} \frac{r^i}{i!}.
\end{align*}
Consider each term in the summation, by Lemma~\ref{cortail}, $r=10s_0t+10s_0$,  and $\umax = 2s_0t+2s_0-1$, for $1\leq u\leq{\umax}$,
\begin{align*}
{|g_x(u)|} e^{-r} \sum^{u+1}_{i=0} \frac{r^i}{i!}
& = \left(\frac{t}{t-1}\right)^u \Pr(\Poi(r) \leq u+1) \left|f\left(\frac{u}{nt}\right)\right|\\
& \leq \left(\frac{t}{t-1}\right)^u \Pr(\Poi(r) \leq 2s_0t+2s_0) \frac{3s_0}{n}\lipf\Paren{\frac{1}{nt}}\\
& \leq \left(\frac{t}{t-1}\right)^u  e^{-4.78(s_0t+s_0)} \frac{3s_0}{n}\lipf\Paren{\frac{1}{nt}}.
\end{align*}
Hence, 
\begin{align*}
 \lambda_x(t-1) \sum^{\umax}_{u=1}{|g_x(u)|} e^{-r} \sum^{u+1}_{i=0} \frac{r^i}{i!}
&\leq \lambda_x(t-1) e^{-4.78(s_0t+s_0)}\frac{3s_0}{n}\lipf\Paren{\frac{1}{nt}} \sum^{\umax}_{u=1}\left(\frac{t}{t-1}\right)^u\\
&\leq \frac{\lambda_x}{n}\lipf \left( \frac{1}{nt}\right) \Paren{(t-1)^2 {3s_0}}  e^{-4.78(s_0t+s_0)}\left(\frac{t}{t-1}\right)^{2s_0t+2s_0}.
\end{align*}
Note that $t>2.5$ yields $\frac{t}{t-1}\leq{e^{0.64}}$ and thus
\begin{align*}
 e^{-4.78(s_0t+s_0)}   \left(\frac{t}{t-1}\right)^{2s_0t+2s_0}
& \leq e^{-4.78(s_0t+s_0)}   e^{1.28(s_0t+s_0)}\\
& = e^{-3.5(s_0t+s_0)}.
\end{align*}
Furthermore,
\begin{align*}
\Paren{(t-1)^2 3s_0} e^{-3.5(s_0t+s_0)} 
& = \Paren{e^{-1.5s_0t}(t-1)^2}\Paren{e^{-3.5s_0}3s_0} e^{-2s_0t} \\
& \leq e^{-2s_0t},
\end{align*}
which completes the proof.
\end{proof}
Analogously, applying the second half of Lemma~\ref{boundbessel}, i.e., $|f_u(y)|\leq{1}$, we get the following alternative upper bound.
\begin{Lemma}\label{boundapprox}
For $t>2.5$ and $s_0\geq 1$,
\begin{align*}
\lvert \hat{h}_x(\lambda_x)-h_x(\lambda_x)\rvert&\leq \frac{1}{n}\lipf \left( \frac{1}{nt}\right)e^{-2s_0t}.
\end{align*}
\end{Lemma}
Lemma~\ref{boundbessel1} and Lemma~\ref{boundapprox} together yield the following upper bound. 
\begin{Lemma}\label{lemmabiasfS}
For $t>2.5$ and $s_0\geq 1$,
\begin{align*}
\text{Bias}(f^*_S)^2&\leq {\Paren{1\land \frac{k^2}{n^2}}} e^{-4s_0t} \lipf^2 \left( \frac{1}{nt}\right).
\end{align*}
\end{Lemma}
\subsection{Bounding \texorpdfstring{$\EE[C^2]$}{E[C2]}}\label{boundCsummary}
Combining all the previous results, for the set of parameters specified in Section~\ref{assum}, if $c_1\sqrt{c_2}\leq{1/11}$, $t>{2.5}$, $n\geq150$, and $1 \leq s_0 \leq \log^{0.2} n$,
\begin{align*}
 \EE[C^2] &\leq \Var(f^*_S)+\Paren{1+{\log n}}\text{Bias}(f^*_S)^2+\Paren{1+\frac{1}{\log n}}R_f^2\\
 &\leq {\Paren{1\land \frac{k}{n}}}  \frac{9s^2_0}{n^{0.22}}\lipf^2 \left( \frac{1}{nt}\right)+(1+\log n){{\Paren{1\land \frac{k^2}{n^2}}} } e^{-4s_0t} \lipf^2 \left( \frac{1}{nt}\right)\\ 
 & +\Paren{1+\frac{1}{\log n}}\Paren{{\Paren{7.1{\Paren{1\land \frac{k}{n}}}  \lipf \left( \frac{1}{n}\right)  e^{-0.3s_0} \log n}}^2+\Paren{\frac{7.1}{n^{3.8}}\lipf \left( \frac{1}{n}\right)}^2}\\
 &\leq 8^2{\Paren{1\land \frac{k}{n}}} \lipf^2 \left( \frac{1}{nt}\right)\log^2 n\Paren{\frac{1}{e^{0.6s_0}}+\frac{1}{n^{0.22}}}+\Paren{\frac{8}{n^{3.8}}\lipf \left( \frac{1}{n}\right)}^2\\
 &\leq 13^2{\Paren{1\land \frac{k}{n}}} \lipf^2 \left( \frac{1}{nt}\right)\Paren{\frac{\log^2 n}{e^{0.6s_0}}}
\end{align*}

\section{Main Results}\label{mainresult}
 To summarize, for properly chosen parameters and sufficiently large $n$, 
 \[
\EE[A^2] \leq\frac{1+T(n)}{nt}\lipf^2\Paren{\frac{1}{nt}}+\left(1+\frac{1}{T(n)}\right)L_{f^E}(p,nt),
\]
\[
\EE[B^2] \leq (8S_f) ^2\Paren{\frac{1}{s_0}\land \frac{k}{n}}+10\lipf^2\left( \frac{1}{nt}\right)\frac{s_0}{n},
\]
and
\[
\EE[C^2] \leq 13^2\Paren{1\land \frac{k}{n}}\lipf^2 \left( \frac{1}{nt}\right)\Paren{\frac{\log^2 n}{e^{0.6s_0}}},
\]
where $T$ is an arbitrary positive function over $\mathbb{N}$.
Furthermore, Cauchy-Schwarz inequality implies 
\[
(f^*(X^N,X^{N'})-f(p))^2=(A+B+C)^2 \leq {(T(n)(C+B)^2+A^2) \left(1+\frac{1}{T(n)} \right)}.
\]
Choosing $T(n)=\log^{\epsilon}n$, the estimation loss of $f^*$ is thus bounded by
\begin{align*}
L_{f^*}(p,2n)&=\EE[(f^*(X^{N}, X^{N'})-\pf)^2]\\
&\leq \EE\left[(\log^{\epsilon}n (C+B)^2+A^2)\left(1+\frac{1}{\log^{\epsilon}n}\right)\right] \\
&\leq 2(1+\log^{\epsilon}n)(\EE[C^2]+\EE[B^2])+ \left(1+\frac{1}{\log^{\epsilon}n}\right)\EE[A^2]\\
&\leq 2(1+\log^{\epsilon}n)\Paren{\EE[C^2]+\EE[B^2]+\frac{1+\log^{\epsilon}n}{2nt\log^{\epsilon}n}\lipf^2\Paren{\frac{1}{nt}}}\\&+ \left(1+\frac{1}{\log^{\epsilon}n}\right)L_{f^E}(p,nt).
\end{align*}
For any property $f$ and set of parameters that satisfy the assumptions in Section~\ref{assum}, 
\[
\EE[C^2]+\EE[B^2]+\frac{1+\log^{\epsilon}n}{2nt\log^{\epsilon}n}\lipf^2\Paren{\frac{1}{nt}}\leq C'_f\min\left\{ \frac{k}{n}+\mathcal{\tilde{O}}\Paren{\frac{1}{n}}, \frac{1}{\log^{2\epsilon}{n}}  \right\},
\]
where $C'_f$ is a fixed constant that only depends on $f$.

Setting $c_1=1$ yields Theorem {\bf $1$} with $C_f=4C'_f$.

In Theorem~\ref{thm1}, for fixed $n$, as $\epsilon\to0$,
the final slack term $1/\log^\epsilon n$ approaches a constant. 
For certain properties it can be improved. 
For normalized support size, normalized support coverage,
and distance to uniformity, a more involved estimator improves this term to
\[
C_{f,\gamma}\min\left\{\frac{k}{n\log^{1-\epsilon}n}+\frac{1}{n^{1-\gamma}}, \frac{1}{\log^{1+\epsilon}{n}}  \right\},
 \]
for any fixed constant $\gamma\in{(0,1/2)}$. 

For Shannon entropy, correcting the bias of
$f^*_L$
and further dividing the probability regions, 
reduces the slack term even more, to
\vspace{-1em}
\[
C_{f,\gamma}\min\left\{\frac{k^2}{n^2\log^{2-\epsilon}n}+\frac{1}{n^{1-\gamma}}, \frac{1}{\log^{2+2\epsilon}{n}}\right\}.
\]
\vspace{-2em}
\section{Experiments}\label{experimentalresults}
We demonstrate the new estimator's efficacy by applying it
to several properties and distributions, and 
comparing its performance to that of several recent estimators~\cite{mmentro,mmsize,mmcover,pnas,jvhw}.
As outlined below, the new estimator was essentially
the best in almost all experiments.
It was out-performed, essentially only by PML,
and only when the distribution is close to uniform.  

\subsection{Preliminaries}
We tested five of the properties outlined in the introduction section:
\ignore{
distance to uniformity, 
normalized support size,
normalized support coverage,
Shannon entropy, and
power sums or equivalently R\'enyi entropy. 
}
Shannon entropy,
normalized support size,
normalized support coverage,
power sums or equivalently R\'enyi entropy, and
distance to uniformity. 
For each of the five properties, we tested the estimator on 
the following six distributions.
a distribution randomly generated from Dirichlet prior with parameter 2;
uniform distribution;
Binomial distribution with success probability $0.3$;
geometric distribution with success probability $0.99$;
Poisson distribution with mean $3{,}000$;
Zipf distribution with power $1.5$.
All distributions had support size $k=10{,}000$.
The Geometric, Poisson, and Zipf distributions were truncated at $k$
and re-normalized. Note that the parameters of the Geometric and
Poisson distributions were chosen so that the expected value would
be fairly large. 

We compared the estimator's performance with $n$ samples to that
of four other recent estimators as well as 
the empirical estimator with $n$, $n\sqrt{\log n}$, and 
$n\log n$ samples.

The graphs denotes
NEW by $f^*$,
$f^E$ with $n$ samples by Empirical,
$f^E$ with $n\sqrt{\log{n}}$ samples by Empirical+,
$f^E$ with $n\log{n}$ samples by Empirical++,
the pattern maximum likelihood estimator in~\cite{mmcover} by PML,
the Shannon-entropy estimator in~\cite{jvhw} by JVHW,
the normalized-support-size estimator in~\cite{mmsize} and 
the entropy estimator in~\cite{mmentro} by WY,
and the smoothed Good-Toulmin Estimator for normalized support
coverage estimation~\cite{pnas}, slightly modified to account for
previously-observed elements that may appear in the subsequent sample,
by SGT. 

While the empirical estimator and the new estimator have the
same form for all properties, as noted in the introduction, the
recent estimators are property-specific, and each was derived
for a subset of the properties. In the experiments we applied these
estimators to the  properties for which they were derived. 
Also, additional estimators~\cite{ventro, pentro, mentro, gsize, ccover, cacover, jcover} for various properties were compared
in~\cite{mmentro,mmsize,pnas,jvhw} and found to perform similarly to or worse than recent 
estimators, hence we do not test them here. 

As outlined in Section \ref{newf}, the new estimator $f^*$ uses
two key parameters $t$ and $s_0$ that determine and all other parameters.
To avoid over-fitting, the data sets used to determine $t$ and
$s_0$ was disjoint from the one used to generate the results shown. 

\begin{table}[h]
  \caption{Values of $t$ and $s_0$ for different properties}
  \label{sample-table1}
  \centering
  \begin{tabular}{llll}
    \toprule
    \cmidrule(r){1-2}
    Property & $t$ & $s_0$\\ 
    \midrule
    Shannon Entropy & $2\log^{0.8}n+1$  & $16\log^{0.2}n$ \\ 
    Normalized Support Size & $\log^{0.7}n+1$ & $16\log^{0.2}n$ \\ 
    Normalized Support Coverage & $\log^{0.8}n+1$ & $8\log^{0.2}n$ \\ 
    Power Sum (0.75) & $\log^{1.0}n+1$ & $4\log^{0.2}n$ \\ 
    Distance to Uniformity& $\log^{0.7}n+1$ & $4\log^{0.2}n$ \\ 
    \bottomrule
  \end{tabular}
\end{table}

Due to the nature of our worst-case analysis and the universality of our results over all possible distributions, we only proved that $f^*$ with $n$ samples works as well as $f^E$ with $n\sqrt{\log n}$ samples. In practice, we chose the amplification parameter $t$ as $\log^{1-\alpha}n+1$,
where $\alpha\in\{0.0,0.1,0.2,...,0.6\}$ was selected based on independent
data, and similarly for $s_0$. Since $f^*$ performed even better than Theorem~\ref{thm1} guarantees,
$\alpha$ ended up between 0 and 0.3 for all properties, 
indicating amplification even beyond $n\sqrt{\log n}$. Finally, to compensate the increase of $t$, in the computation of each coefficient $h_{x,v}$ we substituted $t$ by $\max\left\{{t}/{1.5^{v-1}}, 1.5\right\}$. 

\subsection{Experimental Results}
With the exception of normalized support coverage, all other 
properties were tested on distributions of support size $k=10{,}000$
and number of samples, $n$, ranging from $1{,}000$ to $100{,}000$.
Each experiment was repeated 100 times and the reported results
reflect their mean squared error (MSE). The distributions shown in 
the graphs below are arranged in decreasing order of uniformity. 
In all graphs, the vertical axis is the MSE over the 100 experiments, and the horizontal axis is $\log(n)$.

\vfill
\pagebreak

\subsection*{Shannon Entropy}
For the Dirichlet-drawn and uniform distributions, all recent estimators
outperformed the empirical estimator, even when it was used with
$n\log n$ samples. The best estimator depended on the distribution,
but the new estimator $f^*$ performed best or
essentially as well as the best for all six distributions. 
\vspace{-2em}
\begin{figure*}[h]
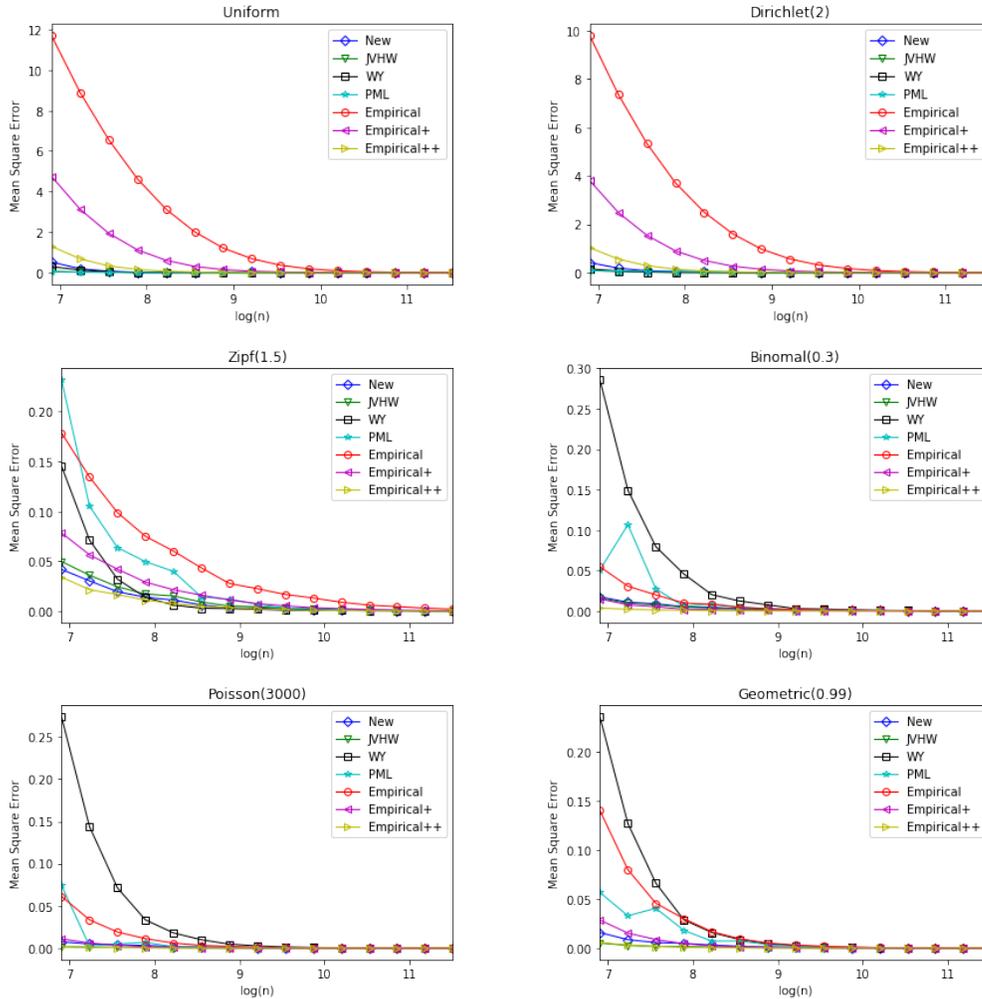

\begin{multicols}{2}
    \includegraphics[width=0.9\linewidth]{eu}\par 
    \includegraphics[width=0.9\linewidth]{ed}\par 
\end{multicols}
\vspace{-2em}
\begin{multicols}{2}
    \includegraphics[width=0.9\linewidth]{ez}\par
    \includegraphics[width=0.9\linewidth]{eb}\par
\end{multicols}
\vspace{-2em}
\begin{multicols}{2}
    \includegraphics[width=0.9\linewidth]{ep}\par
    \includegraphics[width=0.9\linewidth]{eg}\par
\end{multicols}
\vspace{-2em}
\caption{Shannon Entropy}
\end{figure*}
\vfill
\pagebreak

\subsection*{Normalized Support Size}
For the Dirichlet-drawn and uniform distributions, PML and the empirical estimators
were best for small $n$, with the new estimator next. 
For the remaining four distributions, 
empirical with $n\log n$ samples was best, but 
among all estimators using $n$ samples and even empirical with
$n\sqrt{\log n}$ samples, the new estimator was best.

\begin{figure*}[h]
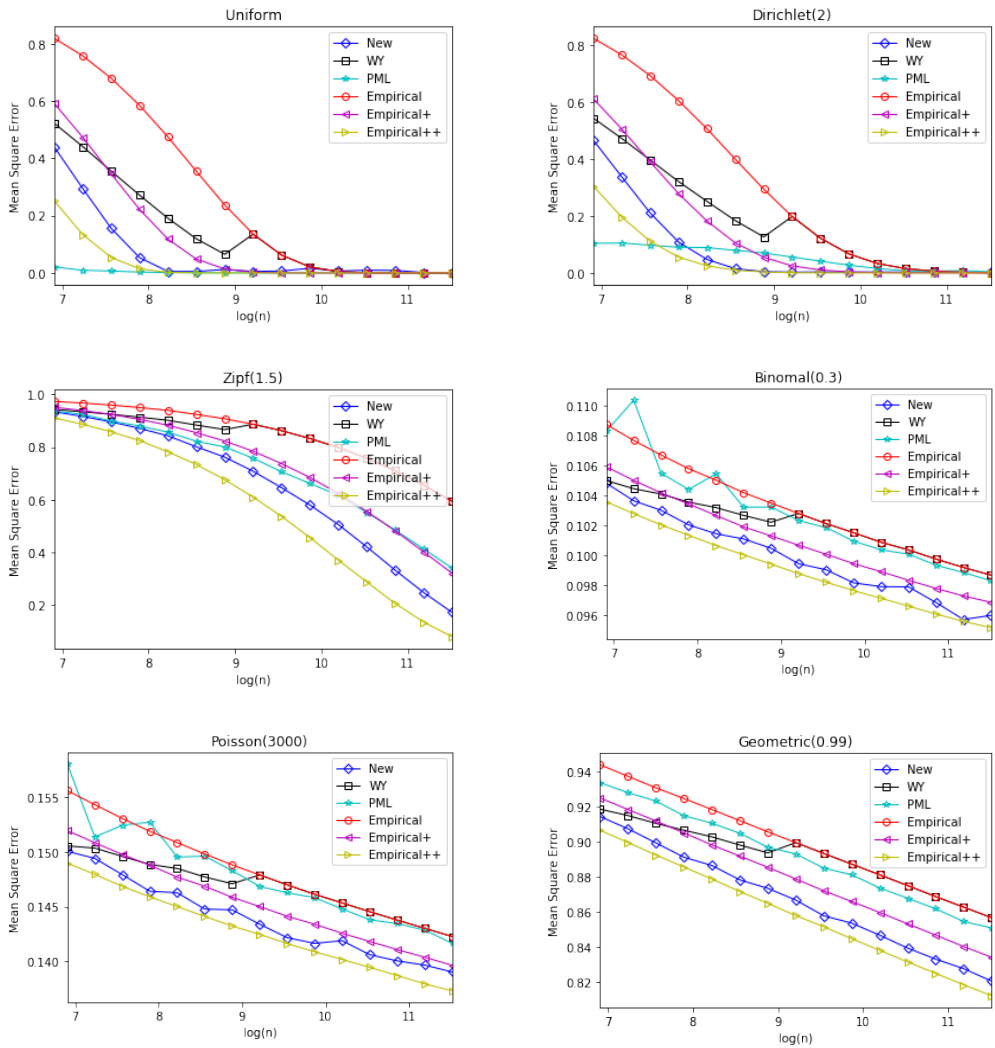

\begin{multicols}{2}
    \includegraphics[width=0.9\linewidth]{su}\par 
    \includegraphics[width=0.9\linewidth]{sd}\par 
     \end{multicols}
\begin{multicols}{2}
    \includegraphics[width=0.9\linewidth]{sz}\par
    \includegraphics[width=0.9\linewidth]{sb}\par
\end{multicols}
\begin{multicols}{2}
    \includegraphics[width=0.9\linewidth]{sp}\par
    \includegraphics[width=0.9\linewidth]{sg}\par
\end{multicols}
\caption{Normalized Support Size}
\end{figure*}
\vfill
\pagebreak

\subsection*{Normalized Support Coverage}
For this property the 
parameter $m$ was set to $5{,}000$. All the distributions have support size $k=1{,}000$ and $n$, the number of samples, ranges from $1{,}000$ to $3{,}000$.
The new estimator was essentially best for all distributions.

\begin{figure*}[h]
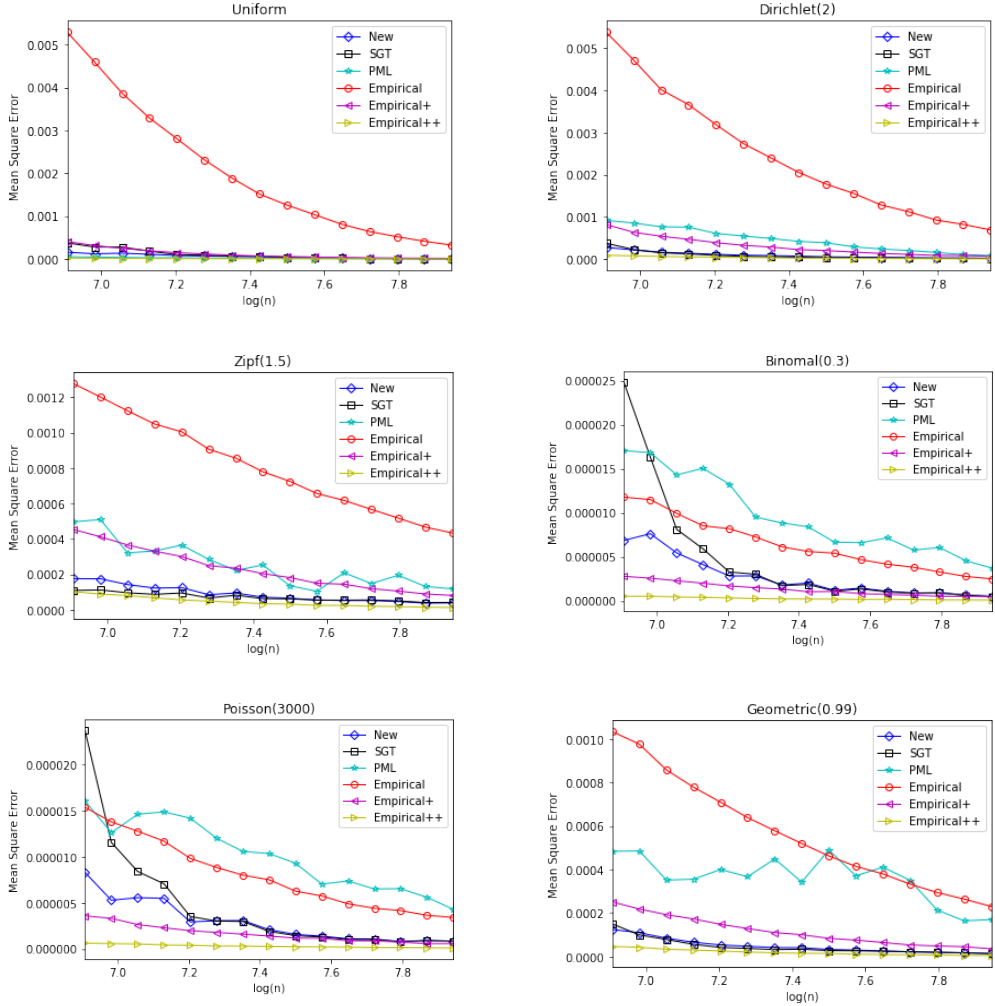

\begin{multicols}{2}
    \includegraphics[width=0.9\linewidth]{cu}\par 
    \includegraphics[width=0.9\linewidth]{cd}\par 
     \end{multicols}
\begin{multicols}{2}
    \includegraphics[width=0.9\linewidth]{cz}\par
    \includegraphics[width=0.9\linewidth]{cb}\par
\end{multicols}
\begin{multicols}{2}
    \includegraphics[width=0.9\linewidth]{cp}\par
    \includegraphics[width=0.9\linewidth]{cg}\par
\end{multicols}
\caption{Normalized Support Coverage}
\end{figure*}
\vfill
\pagebreak

\subsection*{Power Sum (0.75), or equivalently R\'enyi entropy with
  parameter 0.75}
Again PML was best for the Dirichlet-drawn and uniform distributions, however, its performance was not as stable as $f^*$. The new estimator performed as well as $f^E$ with $n\sqrt{\log n}$ samples in all cases and matched $f^E$ with $n{\log n}$ samples for half of the distributions.

\begin{figure*}[h]
\begin{multicols}{2}
    \includegraphics[width=0.9\linewidth]{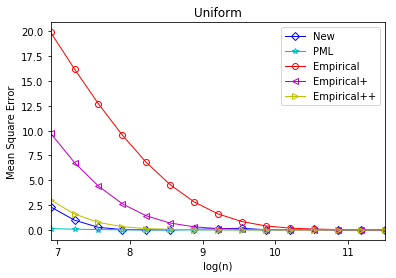}\par 
    \includegraphics[width=0.9\linewidth]{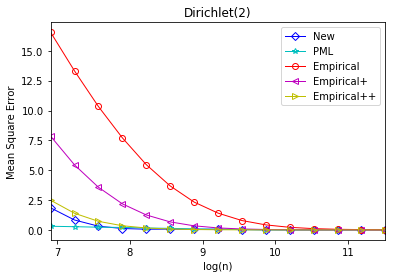}\par 
     \end{multicols}
\begin{multicols}{2}
    \includegraphics[width=0.9\linewidth]{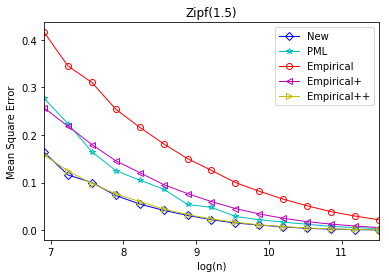}\par
    \includegraphics[width=0.9\linewidth]{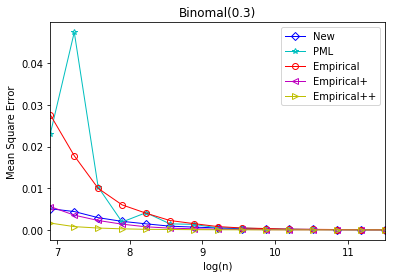}\par
\end{multicols}
\begin{multicols}{2}
    \includegraphics[width=0.9\linewidth]{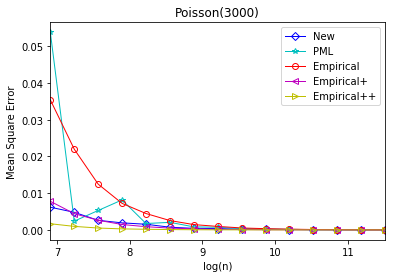}\par
    \includegraphics[width=0.9\linewidth]{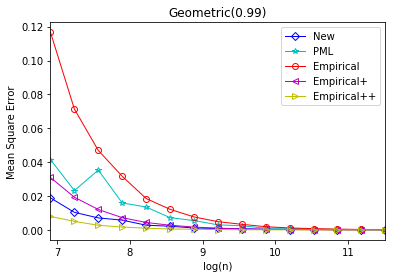}\par
\end{multicols}
\caption{Power Sum (0.75)}
\end{figure*}
\vfill
\pagebreak

\subsection*{Distance to Uniformity}
The new estimator performed as well as $f^E$ with $n\log n$ samples in all cases.
PML was the best estimator for the Dirichlet-drawn and uniform distributions, but provided no improvement 
over the $n$-sample empirical estimator for half of the distributions.

\begin{figure*}[h]
\begin{multicols}{2}
    \includegraphics[width=0.9\linewidth]{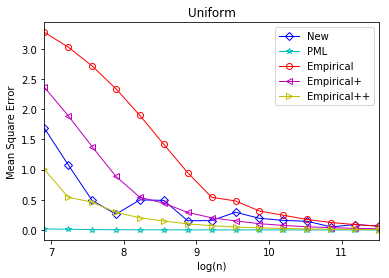}\par 
    \includegraphics[width=0.9\linewidth]{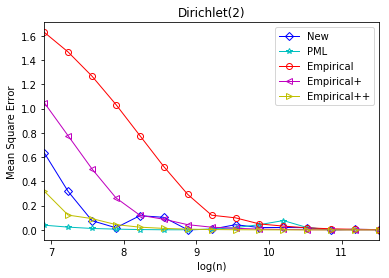}\par 
     \end{multicols}
\begin{multicols}{2}
    \includegraphics[width=0.9\linewidth]{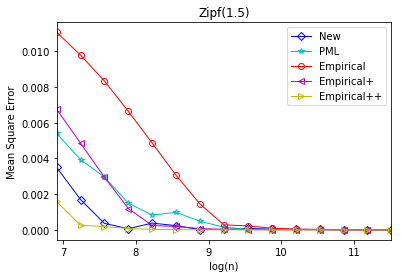}\par
    \includegraphics[width=0.9\linewidth]{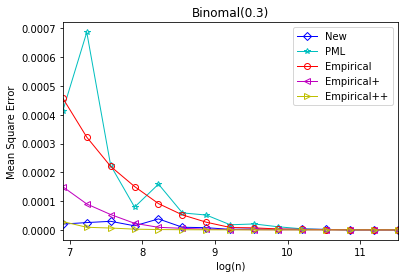}\par
\end{multicols}
\begin{multicols}{2}
    \includegraphics[width=0.9\linewidth]{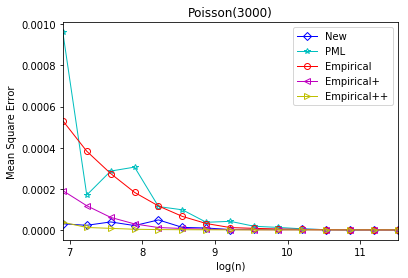}\par
    \includegraphics[width=0.9\linewidth]{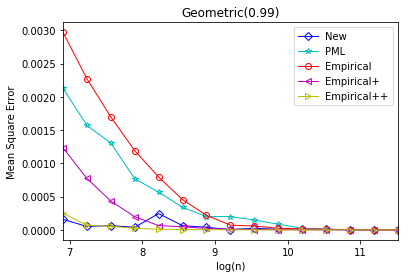}\par
\end{multicols}
\caption{Distance to Uniformity}
\end{figure*}
\vfill
\pagebreak

\end{document}